\documentclass[sigconf,screen]{acmart}
\AtBeginDocument{%
  }

\usepackage{mathtools}
\usepackage{amsfonts}
\usepackage{booktabs}
\usepackage{multirow}
\usepackage{subcaption}
\usepackage{enumitem}
\usepackage{siunitx}
\usepackage{cleveref}
\usepackage{algorithm}
\usepackage{algorithmic}

\newcommand{\unimixer}{RankElastor }
\newcommand{\rankmixer}{RankMixer }

\definecolor{mygray}{RGB}{100,100,100}
\definecolor{mypurple}{RGB}{179,108,161}

\copyrightyear{2026}
\acmYear{2026}
\setcopyright{cc}
\setcctype{by}
\acmConference[KDD '26]{Proceedings of the 32nd ACM SIGKDD Conference on Knowledge Discovery and Data Mining V.2}{August 09--13, 2026}{Jeju Island, Republic of Korea}
\acmBooktitle{Proceedings of the 32nd ACM SIGKDD Conference on Knowledge Discovery and Data Mining V.2 (KDD '26), August 09--13, 2026, Jeju Island, Republic of Korea}
\acmDOI{10.1145/3770855.3818049}
\acmISBN{979-8-4007-2259-2/2026/08}

\begin{document}

\title{Expand More, Shrink Less: Shaping Effective-Rank Dynamics for Dense Scaling in Recommendation}

\author{Guoming Li}
\authornote{Equal contribution}
\orcid{0000-0002-3450-9150}
\affiliation{%
  \institution{The Hong Kong University of Science and Technology (Guangzhou)}
  \city{Guangzhou}
  \country{China}
}
\email{paskardli@outlook.com}

\author{Shangyu Zhang}
\authornotemark[1]
\orcid{0000-0002-8614-6371}
\affiliation{%
  \institution{Tencent Inc.}
  \city{Shenzhen}
  \country{China}
}
\email{vitosyzhang@tencent.com}

\author{Junwei Pan}
\orcid{0009-0003-2697-7012}
\affiliation{%
  \institution{Tencent Inc.}
  \city{Shenzhen}
  \country{China}
}
\email{jonaspan@tencent.com}

\author{Wentao Ning}
\orcid{0000-0002-7571-6957}
\affiliation{%
  \institution{Tencent Inc.}
  \city{Shenzhen}
  \country{China}
}
\email{wentaoning@tencent.com}

\author{Jin Chen}
\orcid{0000-0001-6440-2242}
\affiliation{%
  \institution{Tencent Inc.}
  \city{Shenzhen}
  \country{China}
}
\email{elinjinchen@tencent.com}

\author{Gengsheng Xue}
\orcid{0009-0005-0237-8124}
\affiliation{%
  \institution{Tencent Inc.}
  \city{Shenzhen}
  \country{China}
}
\email{eddiegsxue@tencent.com}

\author{Chao Zhou}
\orcid{0009-0006-5953-4172}
\affiliation{%
  \institution{Tencent Inc.}
  \city{Shenzhen}
  \country{China}
}
\email{derekczhou@tencent.com}

\author{Shudong Huang}
\orcid{0009-0007-6215-272X}
\affiliation{%
  \institution{Tencent Inc.}
  \city{Shenzhen}
  \country{China}
}
\email{ericdhuang@tencent.com}

\author{Haijie Gu}
\orcid{0009-0001-4424-2322}
\affiliation{%
  \institution{Tencent Inc.}
  \city{Shenzhen}
  \country{China}
}
\email{jerrickgu@tencent.com}

\author{Menglin Yang}
\authornote{Corresponding author}
\orcid{0000-0003-2510-5282}
\affiliation{%
  \institution{The Hong Kong University of Science and Technology (Guangzhou)}
  \city{Guangzhou}
  \country{China}
}
\email{menglinyang@hkust-gz.edu.cn}

\renewcommand{\shortauthors}{Guoming Li et al.}

\begin{abstract}
Scaling recommendation models is a central challenge in recommender systems. 
Recently, \textit{\rankmixer} has emerged as an effective solution, operating on a unified token representation and alternating between token mixing and per-token feedforward networks (P-FFNs) to achieve scalable performance. 
However, \rankmixer suffers from \textit{embedding collapse}, where learned representations have low effective rank, limiting expressivity and underutilizing the expanded representation space. 
Through empirical analysis and theoretical insights, we identify rigid token mixing and P-FFN modules as the primary causes of this phenomenon, jointly inducing a \textbf{damped oscillatory trajectory} in effective-rank evolution across layers. 
To address it, we propose \textbf{\unimixer}, a novel architecture that produces spectrum-robust representations with provable collapse mitigation. 
\unimixer introduces two components: (i) \textbf{parameterized full mixing}, which enables expressive token mixing with improved spectral robustness; and (ii) \textbf{GLU-improved P-FFNs}, which stabilize representation spectra through GLU-style FFN modules. 
Extensive experiments on large-scale industrial datasets demonstrate that \unimixer consistently improves recommendation performance, mitigates embedding collapse, and exhibits robust scaling behavior. 
Code is available at this GitHub repository:~\url{https://github.com/vasile-paskardlgm/RankElastor}
\end{abstract}

\ccsdesc[500]{Computing methodologies~Machine learning}

\keywords{Recommender Systems, CTR prediction, Dimensional Collapse}


\maketitle
\newcommand\kddavailabilityurl{https://doi.org/10.5281/zenodo.20252036}
\ifdefempty{\kddavailabilityurl}{}{
\begingroup\small\noindent\raggedright\textbf{Resource Availability:}\\
The source code of this paper has been made publicly available at \url{\kddavailabilityurl}.
\endgroup
}

\section{Introduction}
\label{sec:intro}

Recommender systems are a core application of machine learning, aiming to predict user–item interactions from massive multi-field categorical data~\citep{multi-field-categorical-data}. 
They have become indispensable in modern digital platforms, powering applications such as e-commerce, social media, and content recommendation. 
Recent advances in deep learning–based recommenders enable flexible feature representation learning and complex feature interaction modeling, leading to strong performance in large-scale industrial deployments.

Inspired by the success of large foundation models~\citep{foundation-model-clip,foundation-model-gpt4,foundation-model-latentdiffusion,foundation-model-sam}, scaling model capacity has emerged as a natural direction for recommender systems~\citep{scaling-law-rec-1-wukong,scaling-law-rec-2,scaling-law-rec-3,scaling-law-rec-4}. 
Among recent scaling-oriented architectures, \textbf{\rankmixer}~\cite{rankmixer} has attracted significant attention. 
The \rankmixer architecture can be understood as three core components:
(\textbf{\romannumeral1}) a \textit{tokenization} module that maps heterogeneous embeddings from multiple fields into a unified token representation space;
(\textbf{\romannumeral2}) a \textit{token mixing} module that augments token representations by combining the original token matrix with its block-transposed view; and
(\textbf{\romannumeral3}) a \textit{per-token FFN} (P-FFN) module that models feature interactions using independent feedforward networks applied to each token~\cite{feature-interaction-1,feature-interaction-2}. 
By iteratively alternating the token mixing and P-FFN modules, \rankmixer forms a deep and scalable recommender architecture that achieves \textbf{state-of-the-art} performance on multiple benchmarks as well as practical deployments.

Despite its empirical success, \rankmixer has not been examined from the perspective of \textit{embedding collapse}~\cite{embedding-collapse-1-multiembedding}. 
Embedding collapse has recently been identified as a fundamental obstacle in scaling recommender systems~\citep{embedding-collapse-1-multiembedding,embedding-collapse-2,embedding-collapse-3,embedding-collapse-4-ulrec,feagen}, where learned representations concentrate in low-rank subspaces, limiting representation diversity and reducing the benefits of model scaling. 
While prior work has attempted to mitigate collapse via architectural modifications~\cite{embedding-collapse-1-multiembedding} or optimization strategies~\cite{CovLoss,embedding-collapse-useME-3}, these solutions are typically designed for specific recommenders. 
A systematic understanding of embedding collapse in \rankmixer remains missing.

In this paper, we investigate \rankmixer through the lens of \textit{effective rank}~\cite{effective-rank,norm-erank-1}, a spectral measure that captures representation capacity beyond algebraic rank and generalizes prior collapse analysis tools in recommenders~\cite{embedding-collapse-1-multiembedding,feagen}. 
Empirically, we observe that \rankmixer exhibits a distinctive \textbf{damped oscillatory trajectory} in rank evolution across layers: token-mixing modules slightly increase effective rank, while P-FFN modules contract it. 
Although this alternating behavior yields modest improvements over conventional recommenders (e.g., DCNv2~\cite{dcnv2+crossnet}), which typically show monotonic rank decay, the gains remain limited and do not reliably prevent representation collapse as model depth increases. 
Complementing these observations, our theoretical analysis shows that the original token-mixing and P-FFN modules in \rankmixer collectively constrain spectral robustness, explaining why collapse mitigation remains incomplete.

To tackle this drawback and unlock the potential of the token-transformation-based recommenders, we propose \unimixer, a novel architecture designed to produce spectrum-robust representations with provable collapse mitigation. 
\unimixer introduces two key components: (\textbf{\romannumeral1}) \textit{parameterized full mixing}, which performs learnable fine-grained mixing over tokens to improve spectral expressiveness; and (\textbf{\romannumeral2}) \textit{GLU-improved P-FFNs}, which employ gated activations~\cite{GLU-activation} to prevent collapse amplification induced by the original \rankmixer P-FFNs. 
Together, these modules produce more stable representation spectra and improved scaling behavior.

We evaluate \unimixer against strong recommender baselines, including \rankmixer, on industrial-scale benchmarks Criteo~\cite{criteo} and Avazu~\cite{avazu}, enabling a practically aligned comparison. 
Experimental results show that \unimixer consistently improves downstream CTR prediction performance, achieving over \textbf{0.001 AUC gain} against the strongest baseline — a superior improvement according to~\cite{DIN}. 
Beyond accuracy improvements, \unimixer produces representations with markedly higher effective rank, indicating stronger mitigation of representation collapse. 
In addition, \unimixer exhibits markedly better parameter scaling behavior than \rankmixer, confirming its advantage as a scalable recommender architecture.
Our contributions are summarized below:

\begin{itemize}[leftmargin=15pt,parsep=2pt,itemsep=2pt,topsep=2pt]
    \item We analyze \rankmixer from the perspective of embedding collapse, providing both empirical evidence and theoretical justification showing that collapse remains insufficiently addressed in the architecture.
    \item We propose \unimixer, a novel deep recommender architecture that mitigates representation collapse through parameterized full mixing and GLU-improved P-FFNs, with theoretical guarantees on spectral robustness.
    \item We conduct extensive experiments on industrial-scale benchmarks demonstrating that \unimixer improves recommendation performance, representation diversity, and parameter scaling behavior, providing a new direction for scaling recommender architectures.
\end{itemize}
\section{Research Background}
\label{sec:background}

\subsection{Notations and Preliminaries}\label{sec:background:prelim}

Recommendation models aim to predict user actions based on features drawn from multiple fields~\cite{multi-field-categorical-data}. 
In line with the application scenario considered in this paper, namely \textit{CTR/ranking prediction}~\cite{ctr-prediction}, we consider \(n\) fields, where the \(i\)-th field is denoted as \(\mathcal{X}_i\), and define the joint feature space as \(\mathcal{X}=\mathcal{X}_1\times\mathcal{X}_2\times\dots\times\mathcal{X}_n\). 
Let \(\mathcal{Y}\) denote the prediction space; the goal of a recommendation model is to learn a mapping from \(\mathcal{X}\) to \(\mathcal{Y}\).

We focus on recommenders following the ``embedding-interaction'' architecture. 
Such models first transform an input sample \(X \in \mathcal{X}\) into a \textit{\textbf{feature embedding}} \(E \in \mathbb{R}^{n \times k}\), where \(k\) is the embedding dimension and the \(i\)-th row \(E_i\) denotes the embedding vector for field \(i\). 
The embedding matrix is then processed by a \textit{\textbf{feature interaction}} module~\cite{feature-interaction-1,feature-interaction-2}, which models correlations across fields and produces informative representations for prediction. 
This architecture is widely adopted in real-world recommender systems due to its strong empirical performance~\cite{dcnv2+crossnet,xdeepfm+cin,autoint,fm}.

\subsubsection*{\bf Embedding Collapse~\cite{embedding-collapse-1-multiembedding}.} In general machine learning settings, dimensional collapse refers to models degenerating into trivial representations that map inputs to nearly constant outputs~\citep{dimension-collapse}, which can be characterized through spectral analysis of learned representations~\citep{dimension-collapse-svd}. 
In recommender systems, a related phenomenon known as \textit{\textbf{embedding collapse}} has recently been identified: the embedding matrix becomes approximately low-rank with multiple near-zero singular values~\citep{embedding-collapse-1-multiembedding,embedding-collapse-2,embedding-collapse-3,embedding-collapse-4-ulrec}. 

This phenomenon is commonly measured using \textit{effective rank}~\cite{effective-rank}, which characterizes the spectral distribution of a matrix beyond algebraic rank. 
Among several formulations, we adopt a widely used norm-based definition~\cite{norm-erank-1,norm-erank-3,norm-erank-4}, also known as the stable rank~\cite{norm-erank-2}:
\begin{align}\label{eq:norm-erank}
\operatorname{erank}(X) = \frac{\sum_{i}\sigma^2_i}{\max_i\sigma^2_i} 
= \frac{\|X\|_F^2}{\|X\|_2^2}\ ,
\end{align}
where \(\sigma_i\) denotes the singular values of \(X\).

\subsubsection*{\bf \rankmixer~\cite{rankmixer}.} As a recent deep recommender architecture, \rankmixer has attracted substantial attention due to its streamlined design, strong empirical performance, and favorable scaling behavior. 
Its architecture consists of three modules: (\textbf{\romannumeral1}) \textit{tokenization}; (\textbf{\romannumeral2}) \textit{token mixing}; and (\textbf{\romannumeral3}) \textit{per-token FFNs} (P-FFNs). We briefly summarize them below:

\begin{itemize}[leftmargin=15pt,parsep=2pt,itemsep=2pt,topsep=2pt]
    \item {\it \underline{Tokenization}.} The tokenization module takes the embedding matrix \(E\) and groups embedding vectors into \(T\) clusters, followed by a unified projection (e.g., linear mapping or MLP) into a shared \(D\)-dimensional space, producing a token matrix \(X^{(0)}\in\mathbb{R}^{T\times D}\). 
    Each row \(X^{(0)}_t\in\mathbb{R}^D\) is called a token embedding. 
    This module is primarily used for aligning embeddings from heterogeneous sources. 
    In this paper, we focus on the subsequent core modules.
    
    \item \textbf{\textit{\underline{Token mixing}.}} The token mixing module partitions the column dimension \(D\) into \(H\) equal segments of size \(d = D/H\), forming a \(T \times T\) grid of blocks, each in \(\mathbb{R}^{1 \times d}\). 
    At the \(l\)-th \rankmixer block, a block-transpose operation is first applied to the input \(X^{(l)}\), exchanging block pairs \(X^{(l)}_{i,j}\) and \(X^{(l)}_{j,i}\). 
    The transformed representation is then combined with the original token matrix through a residual connection, followed by output layer normalization, producing enhanced token representations with cross-token information mixing.
    
    \item \textbf{\textit{\underline{P-FFNs}.}} \rankmixer then sends the mixed tokens into the token-specific FFN modeuls, which is called P-FFN. 
    This module serves as the function for performing feature interaction as those in typical ``embedding-interaction'' type recommenders. 
    In \rankmixer, the P-FFN is configured as a two layer with GELU~\cite{GELU-activation} as activation function.
\end{itemize}

Stacking token mixing and P-FFN modules yields a deep and scalable recommender architecture. 
An overview of the \rankmixer architecture is shown in Figure~\ref{fig:arch}(\textbf{a}).

\begin{figure*}[t]
  \centering
    \begin{subfigure}{0.245\textwidth}
        \centering
        \includegraphics[width=\linewidth]{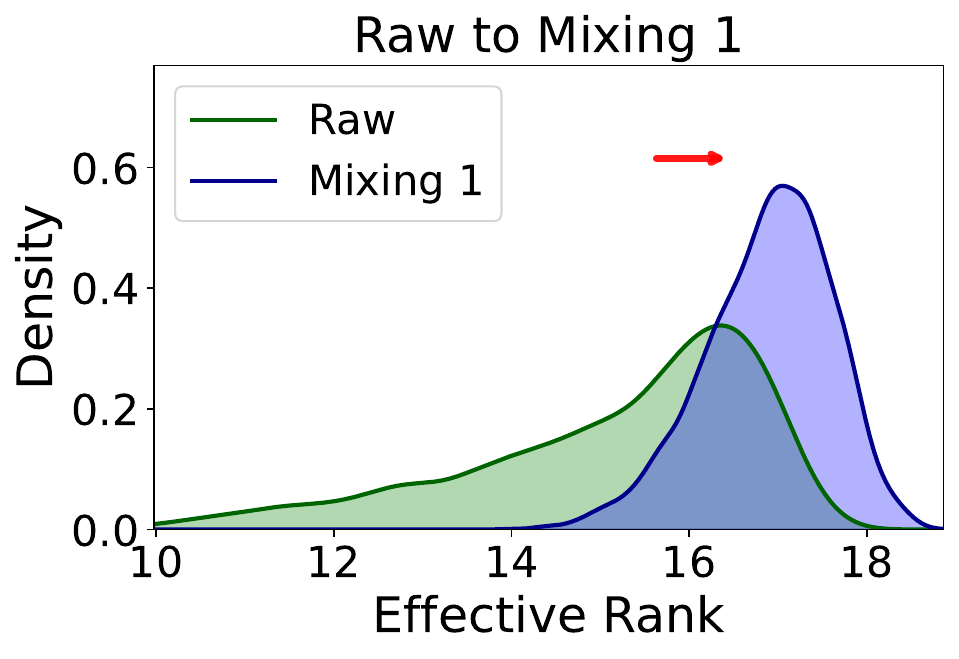}
        \label{fig:erank-criteo-base-1-rankmixer}
    \end{subfigure}
    \begin{subfigure}{0.245\textwidth}
        \centering
        \includegraphics[width=\linewidth]{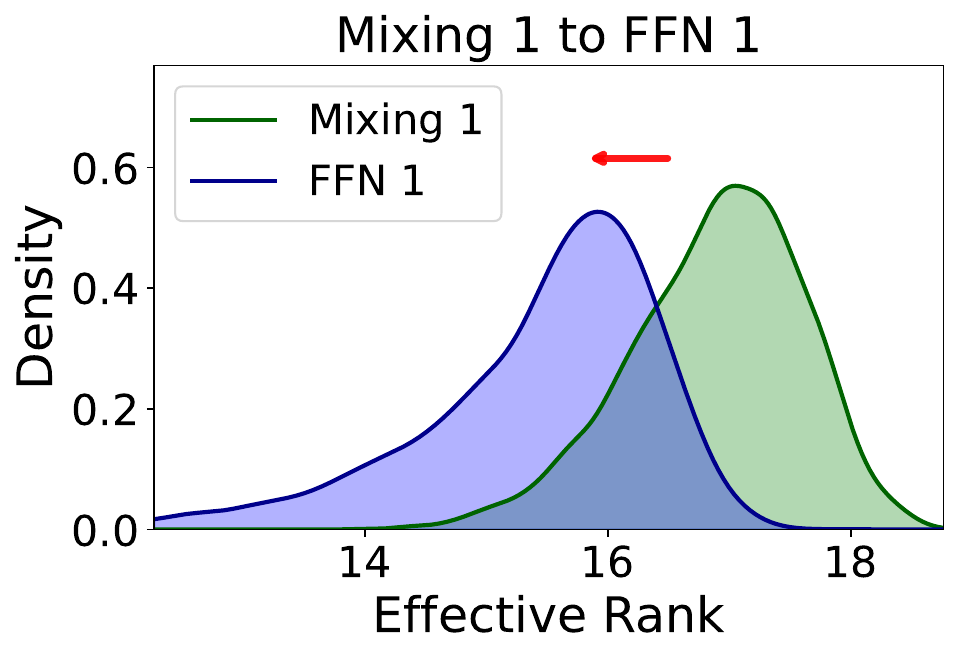}
        \label{fig:erank-criteo-base-2-rankmixer}
    \end{subfigure}
    \begin{subfigure}{0.245\textwidth}
        \centering
        \includegraphics[width=\linewidth]{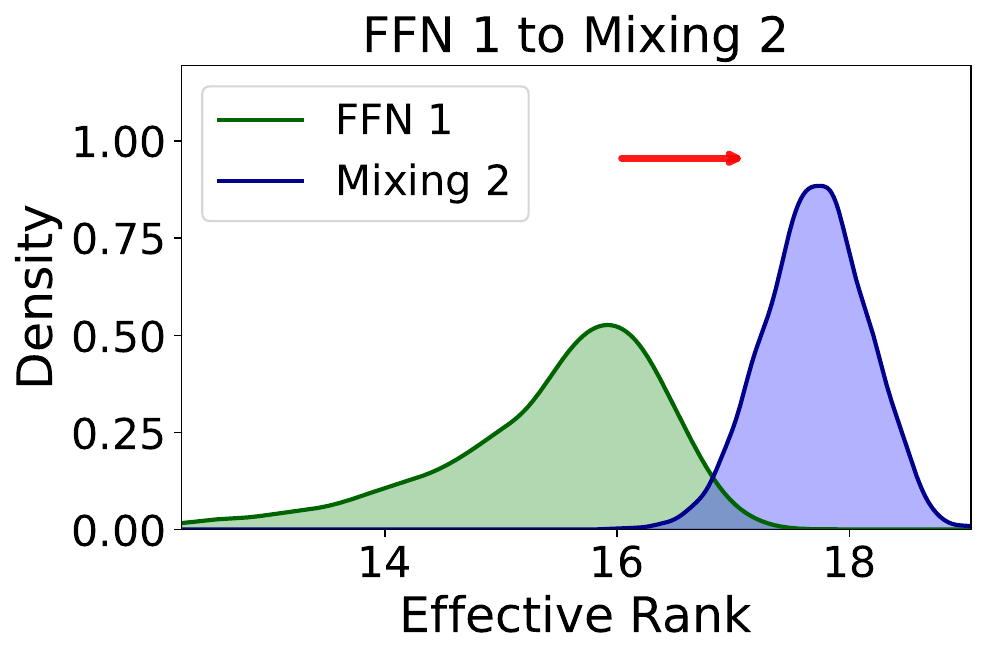}
        \label{fig:erank-criteo-base-3-rankmixer}
    \end{subfigure}
    \begin{subfigure}{0.245\textwidth}
        \centering
        \includegraphics[width=\linewidth]{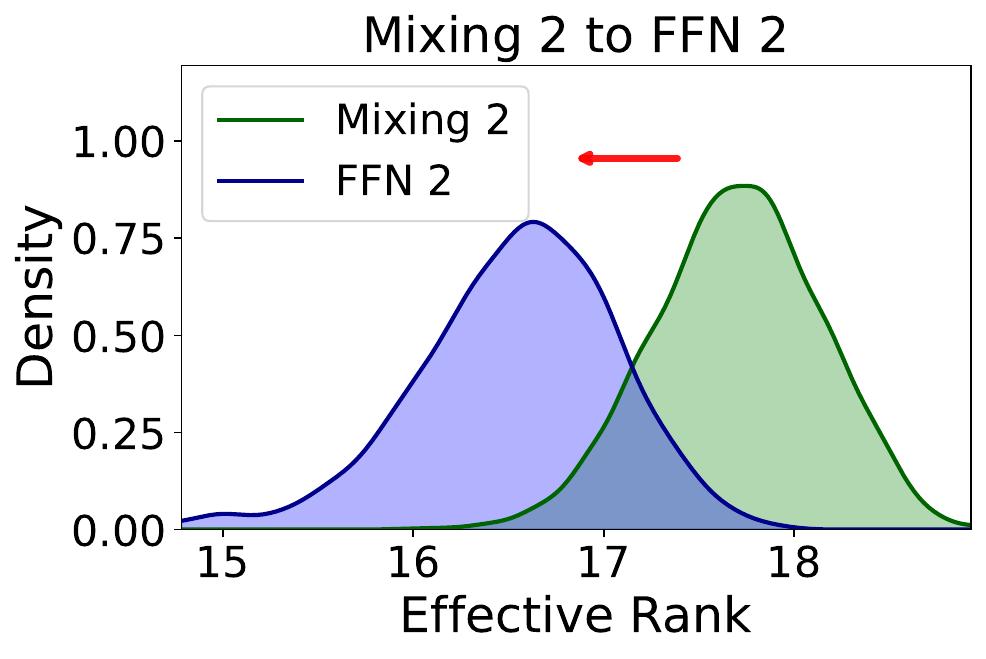}
        \label{fig:erank-criteo-base-4-rankmixer}
    \end{subfigure}
    \\
    \begin{subfigure}{0.245\textwidth}
        \centering
        \includegraphics[width=\linewidth]{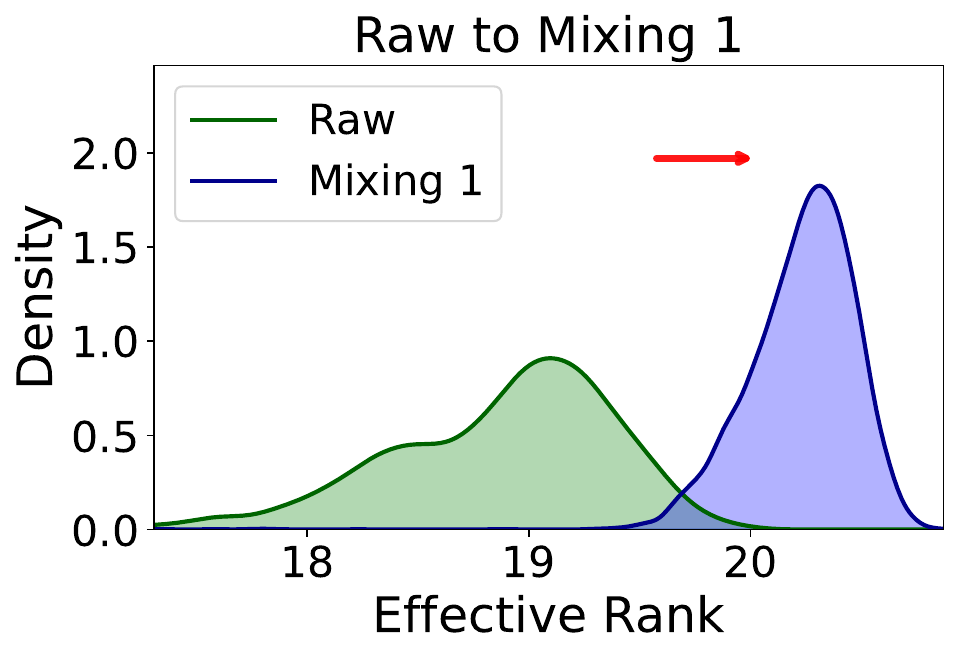}
        \caption{Raw Embedding $\to$ Mixing 1}
        \label{fig:erank-avazu-base-1-rankmixer}
    \end{subfigure}
    \begin{subfigure}{0.245\textwidth}
        \centering
        \includegraphics[width=\linewidth]{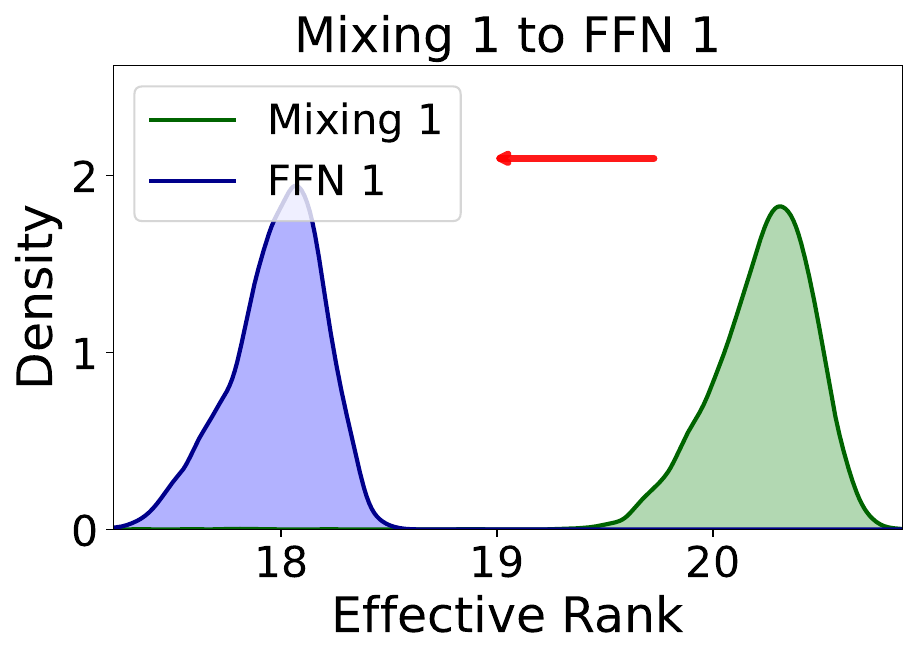}
        \caption{Mixing 1 $\to$ FFN 1}
        \label{fig:erank-avazu-base-2-rankmixer}
    \end{subfigure}
    \begin{subfigure}{0.245\textwidth}
        \centering
        \includegraphics[width=\linewidth]{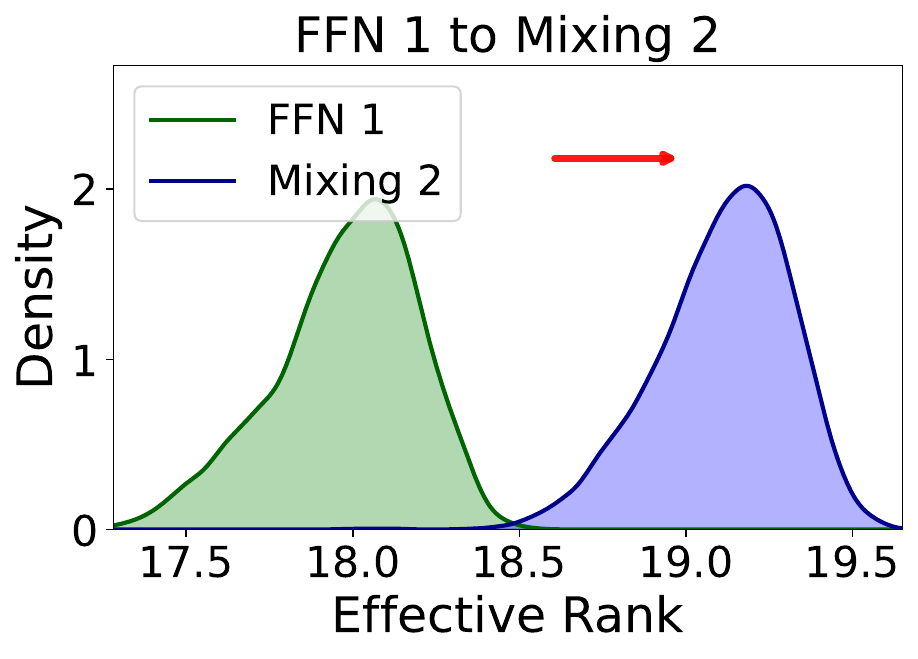}
        \caption{FFN 1 $\to$ Mixing 2}
        \label{fig:erank-avazu-base-3-rankmixer}
    \end{subfigure}
    \begin{subfigure}{0.245\textwidth}
        \centering
        \includegraphics[width=\linewidth]{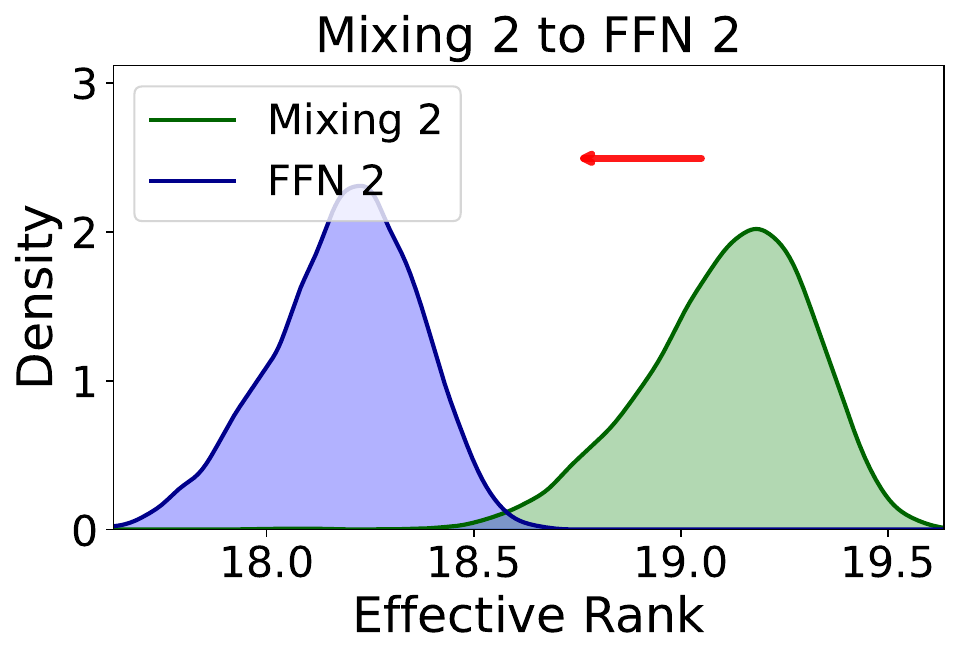}
        \caption{Mixing 2 $\to$ FFN 2}
        \label{fig:erank-avazu-base-4-rankmixer}
    \end{subfigure}
  \caption{Distributional shift of effective rank across representation stages in \rankmixer. Distributions of per-sample effective rank from raw embeddings to successive module outputs on Criteo (top) and Avazu (bottom).}
  \label{fig:erank-rankmixer}
\end{figure*}

\begin{figure}[t]
  \centering
    \begin{subfigure}{0.495\linewidth}
        \centering
        \includegraphics[width=1.02\linewidth]{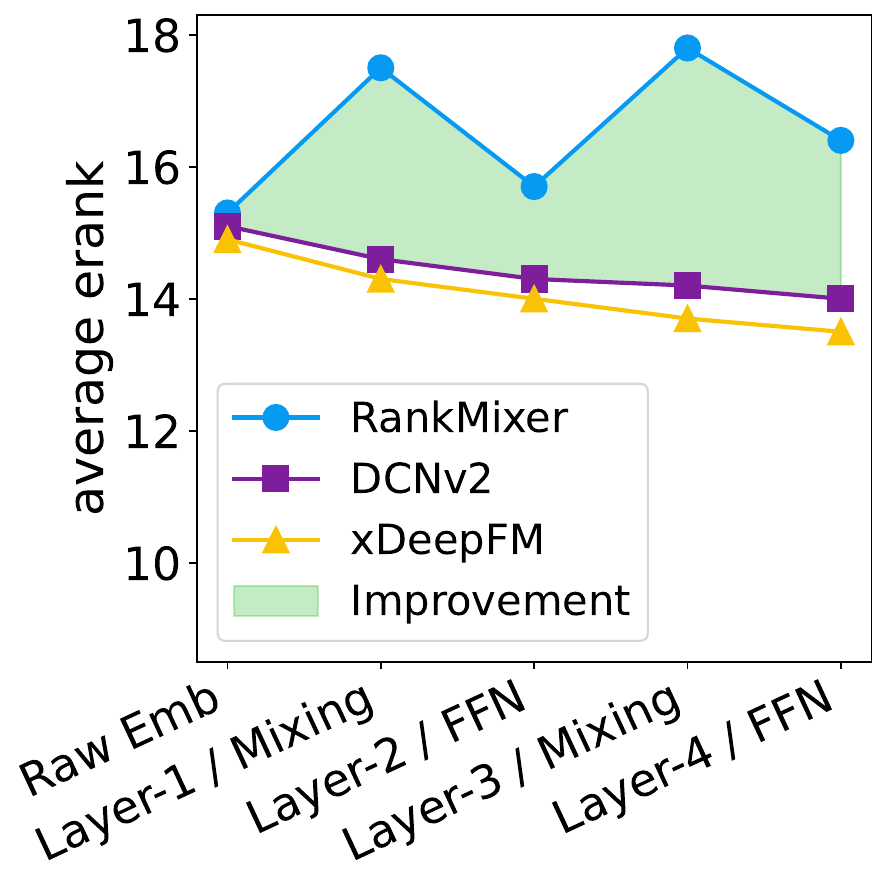}
        \caption{Comparison on Criteo.}
        \label{fig:avg-erank-criteo-rankmixer}
    \end{subfigure}
    \begin{subfigure}{0.495\linewidth}
        \centering
        \includegraphics[width=1.02\linewidth]{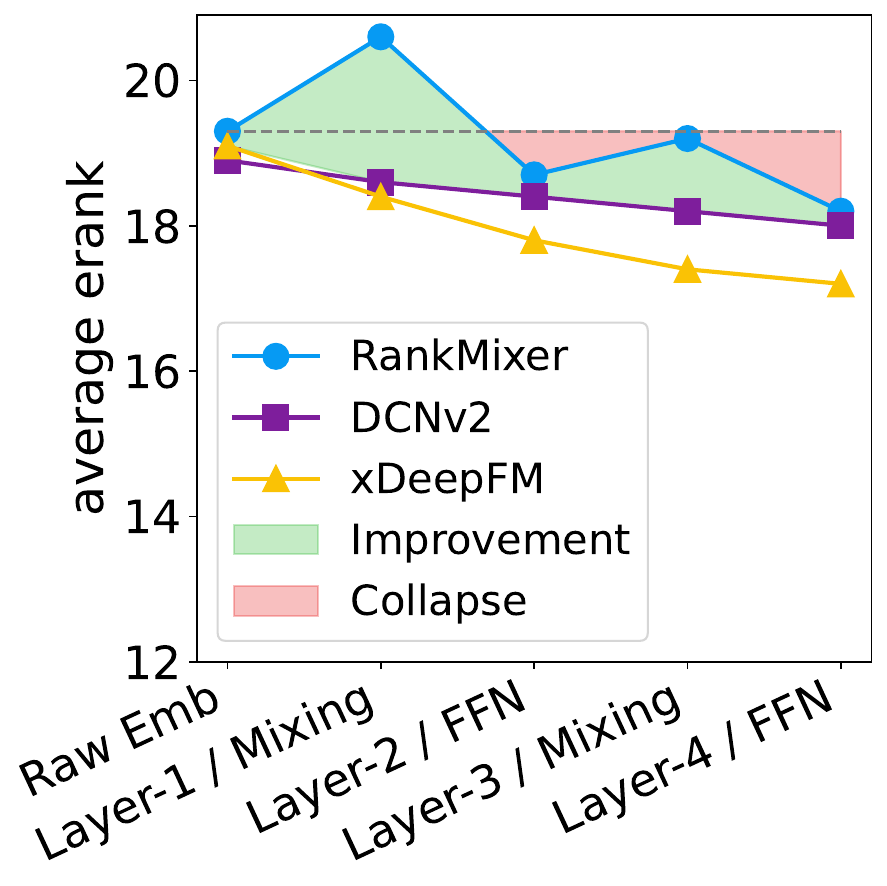}
        \caption{Comparison on Avazu.}
        \label{fig:avg-erank-avazu-rankmixer}
    \end{subfigure}
    %
  \caption{Comparison of average effective rank across representation stages for \rankmixer (damped oscillatory trajectory) and baselines (monotonic decay).}
  \label{fig:avg-erank-rankmixer}
\end{figure}

\subsection{On Embedding Collapse in \rankmixer}\label{sec:background:collapse-rankmixer}

Although \rankmixer demonstrates strong scalability and effectiveness, its behavior under embedding collapse has not been systematically studied. 
Existing collapse analyses primarily focus on conventional recommenders~\cite{embedding-collapse-1-multiembedding,feagen}, such as DCNv2~\cite{dcnv2+crossnet} and xDeepFM~\cite{xdeepfm+cin}. 
In this section, we analyze the effective-rank dynamics of \rankmixer representations across its core modules (token mixing and P-FFNs), combining empirical observations with theoretical justification.

\subsubsection*{\bf Empirical Demonstrations.} We conduct CTR prediction experiments on two industrial-scale benchmarks, Criteo~\cite{criteo} and Avazu~\cite{avazu}, using the FuxiCTR framework~\cite{benchmark-fuxictr}. 
We compare \rankmixer (with two token-mixing blocks and two P-FFN blocks) against DCNv2 and xDeepFM configured with comparable depth, resulting in four-layer recommenders for fair comparison. 
We then visualize the average effective rank of module outputs across test samples. 
All remaining experimental settings follow those described in Section~\ref{sec:exps}.

While full recommendation performance results are reported in Section~\ref{sec:exps}, we focus here on effective-rank dynamics. 
As shown in Figure~\ref{fig:avg-erank-rankmixer}, conventional models exhibit a \textbf{monotonic decay} in effective rank with depth; in contrast, \rankmixer shows an alternating pattern: token mixing increases effective rank, whereas P-FFNs shrink it, producing a \textbf{damped oscillatory trajectory}. 
Although such sawtooth pattern allows \rankmixer to maintain slightly higher effective rank than conventional models, the improvement remains limited and does not eliminate collapse.

On Criteo, the effective rank at the final layer only marginally exceeds that of the raw embeddings, indicating that the rank-enhancement effect of token mixing is modest in practice. 
This limitation becomes more pronounced on Avazu (Figure~\ref{fig:avg-erank-avazu-rankmixer}), where contractions introduced by P-FFNs dominate across layers, causing collapse to re-emerge despite the presence of token mixing.

Since higher effective rank typically indicates better utilization of representation capacity and correlates with improved recommendation performance~\cite{embedding-collapse-1-multiembedding,feagen}, these results reveal both the advantage and limitation of \rankmixer from the perspective of embedding collapse: although the architecture improves representation rank compared with conventional recommenders, it does not fundamentally prevent collapse.


\subsubsection*{\bf Theoretical Justifications.} We further complement the empirical analysis with theoretical justification. 
Both the effective-rank improvement and the remaining collapse limitation of \rankmixer can be traced to its two modules: token mixing and P-FFNs. 
The following theorems provide theoretical justification for these effects.

\begin{theorem}[Effective rank under block-transpose mixing]
\label{thm:erank-bound-blocktranspose}
Let $X\in\mathbb R^{T\times D}$ be a representation matrix with \(D = Td\) and \(d \ge 1\). Partition $X$ into a \( T \times T\) grid of blocks $X_{ij} \in \mathbb{R}^{1\times d}$. Define the block-transpose operator $\mathcal{T}$ by \((\mathcal{T}(X))_{ij}=X_{ji}\), and let $Y = \mathcal{T}(X)$. Denote the algebraic rank and effective rank of $X$ by $\operatorname{rank}(X)=r$ and $\operatorname{erank}(X)=k$. Assume $X$ satisfies Frobenius-orthogonality and spectral incoherence with its block transpose:
\[
\langle X, Y \rangle_F \approx 0 \quad \mathrm{and} \quad \|X+Y\|_{2}^2\approx \max(\|X\|_{2}^2,\|Y\|_{2}^2)\ .
\]
Let $M = X+Y$ and define $\mu = \operatorname{erank}(Y)$. Then
\[
\frac{2k\mu}{(\sqrt{k}+\sqrt{\mu})^2} \le \operatorname{erank}(M) \le 2 (k + \mu)\ ,
\]
where $\mu \le \operatorname{rank}(Y) \le \min\{T, rd\}$.
\end{theorem}

\begin{theorem}[Effective rank failure of standard FFNs]
\label{thm:ffn_failure}
Let $X\in\mathbb R^{T\times D}$ be a representation matrix with effective rank
\(
\operatorname{erank}(X)=k \quad\text{and}\quad \frac{k}{D}\le \gamma,
\)
for some fixed constant $\gamma\in(0,1)$ independent of $D$. 
Consider the per-row feedforward network 
\(
\mathcal F(X)=\phi(XA)B,
\)
where $A\in\mathbb R^{D\times m}$ and $B\in\mathbb R^{m\times D}$ have i.i.d.\ sub-Gaussian entries with variance $1/D$, and $\phi$ is positively homogeneous. 
Then:
\begin{itemize}[leftmargin=15pt,parsep=2pt,itemsep=2pt,topsep=2pt]
    \item (\textbf{Deterministic collapse})  
    If $X$ has an algebraic rank $1$, then
    \(
    \operatorname{rank}(\mathcal F(X))\le 2
    \), and 
    \(
    \operatorname{erank}(\mathcal F(X))\le 2
    \)
    , with equality $=1$ when all rows of $X$ have the same sign.
    \item (\textbf{Probabilistic contraction})  
    If the pre-activations satisfy a nontrivial response-gap condition (See Appendix~\ref{app:thm:ffn_failure}), then with probability at least $1-\exp(-c k)$,
    \(
    \operatorname{erank}(\mathcal F(X)) \le \alpha\,\operatorname{erank}(X),
    \)
    for some constant $\alpha\in(0,1)$ depending only on $\gamma$ and activation $\phi$.
\end{itemize}
\end{theorem}

Theorem~\ref{thm:erank-bound-blocktranspose} explains why token mixing can increase effective rank. 
The block-transpose mixing produces a representation whose effective rank admits the lower bound
\(
\frac{2k\mu}{(\sqrt{k}+\sqrt{\mu})^2}.
\)
showing that mixing introduces a rank-expansion effect. 
However, the theorem also reveals an intrinsic limitation: the achievable improvement is bounded by
\(
\operatorname{erank}(X+Y)\le 2(k+\mu).
\)
Since recommender embeddings are often already spectrally collapsed at initialization stages~\cite{embedding-collapse-1-multiembedding,embedding-collapse-2,embedding-collapse-5}, this bounded expansion typically results in only modest rank improvement in practice. 
This explains the small rank increase observed after token mixing in Figures~\ref{fig:erank-avazu-base-1-rankmixer} and~\ref{fig:erank-avazu-base-3-rankmixer}.

Theorem~\ref{thm:ffn_failure} characterizes the complementary effect of P-FFNs. 
When the input representation is already low-rank, the FFN module can further contract effective rank, either deterministically in the rank-1 case or probabilistically generic rank-\(k\). 
This explains the repeated effective-rank decreases observed in Figures~\ref{fig:erank-avazu-base-2-rankmixer} and~\ref{fig:erank-avazu-base-4-rankmixer}.


\subsubsection*{\bf Summary.} Taken together, the empirical and theoretical analyses reveal a consistent picture of spectra dynamics in \rankmixer. 
Although \rankmixer maintains higher effective rank than conventional recommenders, the improvement is inherently limited and does not reliably prevent collapse from re-emerging across layers. 
This limitation arises from the architecture itself: token mixing provides bounded rank expansion, while P-FFNs exhibit rank-contractive behavior, jointly producing the characteristic damped oscillatory trajectory across depth.

Importantly, this diagnosis suggests a principled direction for architectural refinement — \textbf{Expand More and Shrink Less}, enabling token mixing to produce stronger spectrum expansion while making P-FFNs less spectrum-reductive. 
Motivated by this insight, we propose \unimixer in the next section.
\section{Method: \unimixer}
\label{sec:method}

\begin{figure*}[!t]
  \centering
  \includegraphics[width=\linewidth]{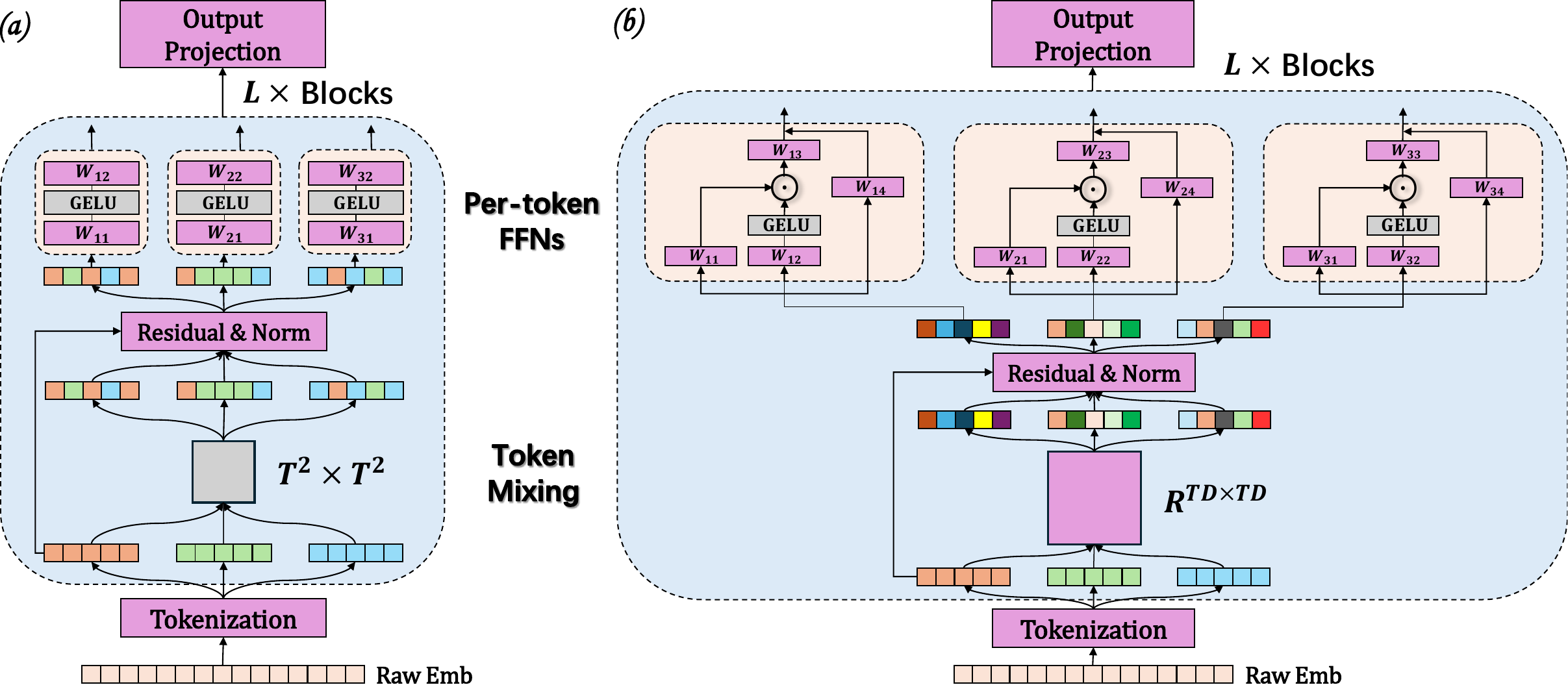}
  \caption{Architectural overview of (\textbf{a}) \rankmixer and (\textbf{b}) \unimixer. \underline{\textcolor{mypurple}{Purple}} modules indicate trainable components, while \underline{\textcolor{mygray}{Gray}} modules indicate non-parametric components. The key difference between the models lies in token-mixing and P-FFNs.}
  \label{fig:arch}
\end{figure*}

In this section, we present \unimixer, a novel deep recommender architecture that generalizes the token-mixing and P-FFN modules of \rankmixer through establishing collapse-robust token transformation. 
We focus on two core architectural components in \unimixer: (\textbf{\romannumeral1}) \textit{Parameterized Full Token Mixing} and (\textbf{\romannumeral2}) \textit{GLU-improved P-FFNs}, both supported by theoretical justification. 
An overview of \unimixer is illustrated in Figure~\ref{fig:arch}(\textbf{b}).


\subsection{Parameterized Full Token Mixing}\label{sec:method:full-mixing}
In \rankmixer, token mixing is implemented as a non-parameteric block-transposed operation. 
This operation is equivalent to
\begin{align}
\label{eq:rankmixer-mixing}
\operatorname{vec}(M^\top) = \operatorname{LN}\big((P\otimes I)\operatorname{vec}(X^\top) + \operatorname{vec}(X^\top)\big)\ ,
\end{align}
where \(P \in \mathbb{R}^{T^2 \times T^2}\) is a permutation operator structured as a \textit{commutation matrix}~\cite{commutation-matrix}, \(I\in\mathbb{R}^{\frac{D}{T}\times\frac{D}{T}}\) denotes the identity matrix, \(\operatorname{vec}\) denotes the \textit{vec operator}~\cite{vec-operator}, \(\operatorname{LN}\) denotes layer normalization, and \(\otimes\) is the \textit{Kronecker product}~\cite{matrix-analysis}.

This representation reveals an important structural property: token mixing is governed by the permutation operator \(P\) and the grid scale \(d = \frac{D}{T}\). 
Modifying either the permutation operator or the grid scale leads to different token-mixing behaviors. 
Motivated by this observation, we generalize Eq.~\ref{eq:rankmixer-mixing} by introducing a learnable mixing operator and reducing the grid scale to the finest resolution. 
This leads to the proposed \textbf{parameterized full mixing}:
\begin{align}
\label{eq:full-mixing}
\operatorname{vec}(M^\top) = \operatorname{LN}\big((W + I)\operatorname{vec}(X^\top)\big)\ ,
\end{align}
where \(W\in\mathbb{R}^{TD\times TD}\) is a learnable mixing matrix and \(I\) preserves residual information flow.

This formulation corresponds to the fine-grained case \(d=1\) with fully parameterized mixing weights, enabling interactions across all token–feature coordinates. 
Such generalization strengthens the token-mixing module by allowing the model to refine representations that may otherwise lead to spectral collapse under restricted block-wise transformations. 
We formalize this expressivity advantage in the following theorem.

\begin{theorem}[Expressivity of parameterized full mixing]
\label{thm:parameterized-block-mixing}
Let $X\in\mathbb R^{T\times D}$ be an input representation matrix and $x=\operatorname{vec}(X)\in\mathbb{R}^{N}$ with $N=TD$. For a divisor $d^{\ast}$ of $N$, partition $x$ into $K=N/d^{\ast}$ blocks $\mathcal{B}=\{\mathbf{b}_1,\dots,\mathbf{b}_K\}$, where $\mathbf{b}_k\in\mathbb{R}^{d^{\ast}}$. A parameterized block-mixing scheme with learnable weights $\mathbf{W}\in\mathbb{R}^{K\times K}$ produces \(\mathbf{y}=(\mathbf{W}\otimes\mathbf{I}_{d^{\ast}})\,\mathbf{x} \) and the residual output \(\mathbf{C} =\Phi_{d^{\ast}}(\mathbf{X};\mathbf{W}) \triangleq \mathbf{X} +\operatorname{reshape}\!\left( (\mathbf{W}\otimes\mathbf{I}_{d^{\ast}})\operatorname{vec}(\mathbf{X}) \right).\) Then for any $d^{\ast}>1$, the set of linear transformations realizable by $\Phi_{d^{\ast}}$ is strictly contained in that of the parameterized full mixing case ($d^{\ast}=1$). In particular, the Kronecker constraint $(\mathbf{W}\otimes\mathbf{I}_{d^{\ast}})$ prevents expressing fine-grained, high-rank coordinate interactions that are attainable when $d^{\ast}=1$.
\end{theorem}

The theorem's justification is presented in Appendix~\ref{app:thm:parameterized-block-mixing}. 
The theorem shows that when the grid scale satisfies \(d^{\ast}>1\), there always exist nontrivial inputs whose token-mixing outputs remain spectrally constrained. 
In contrast, the fine-grained case \(d^{\ast}=1\), corresponding to our parameterized full mixing, removes the Kronecker constraint and enables richer coordinate interactions, improving robustness against representational collapse.


\subsection{GLU-improved P-FFNs}\label{sec:method:geglu-ffn}
Although parameterized full mixing improves the expressivity of token representations at the matrix level—in terms of effective rank—its collapse mitigation effect is primarily statistical rather than deterministic. 
In practical recommendation scenarios, highly skewed or degenerate inputs may still lead to degraded representations. 
To further improve robustness, we refine the second core module of \rankmixer, namely the P-FFN.

We replace the GELU-based FFN with a GLU-style gated feed-forward module (\textbf{bias omitted for clarity}):
\begin{align}
\label{eq:geglu-ffn}
Z_t = \big(\operatorname{GELU}(M_t W_1)\odot(M_t W_2)\big)W_3 + M_t W_r\ ,
\end{align}
where \(W_1, W_2\in\mathbb{R}^{D\times rD}\) are lifting projections, \(W_3\in\mathbb{R}^{rD\times D}\) is the compression projection, \(r\) is the expansion ratio, \(W_r\) is a learnable residual mapping, and $\odot$ denotes the \textit{Hadamard product}~\cite{linear-algebra}. 
This design follows the gated activation principle of GLU~\cite{GLU-activation}, which has been shown to improve expressivity and representation quality in modern deep architectures~\cite{GLU-llm-1,GLU-llm-2,GLU-llm-3}.

We next show that the GLU-improved P-FFN provides stronger effective rank recovery capability than the GELU-based FFN used in \rankmixer.

\begin{theorem}[Effective rank recovery via GLU-improved P-FFNs]
\label{thm:geglu_recovery}
Let $X\in\mathbb R^{T\times D}$ satisfy $\operatorname{erank}(X)=k$ with $k/D\le \gamma$ for a fixed constant $\gamma\in(0,1)$.  
Consider the GLU-improved P-FFN with residual
\[
\mathcal G(X)=\big(\phi(XA)\odot(XC)\big)B + XD,
\]
where $A,C\in\mathbb R^{D\times m}$ and $B\in\mathbb R^{m\times D}$ have i.i.d.\ sub-Gaussian entries with variance $1/D$, and $\odot$ denotes the Hadamard product. 
If the hidden width satisfies $m \ge C k\log D$, then with probability at least $1-\exp(-c k)$:
\begin{itemize}[leftmargin=15pt,parsep=2pt,itemsep=2pt,topsep=2pt]
    \item (\textbf{Algebraic lifting})  
    The multiplicative term induces degree-2 interactions, and
    \(
    \operatorname{rank}\!\big(\phi(XA)\odot(XC)\big)\;\ge\;\min\!\left(D,\frac{k(k+1)}{2}\right).
    \)
    \item (\textbf{Effective rank increase})  
    The output satisfies
    \(
    \operatorname{erank}(\mathcal G(X)) \ge \operatorname{erank}(X) + \delta,
    \)
    for some $\delta>0$ depending only on $\gamma$ and initialization constants.
\end{itemize}
\end{theorem}

The theorem's justification is presented in Appendix~\ref{app:thm:geglu_recovery}. 
This result shows that GLU-improved P-FFNs can statistically maintain spectrum robustness through multiplicative feature interaction, whereas conventional activation-based FFNs may still exhibit nontrivial failure cases that lead to collapse in effective rank. 
Consequently, the proposed P-FFN improves robustness against representational collapse while maintaining the standard feed-forward computation structure used in \rankmixer.


\subsection{Complexity Analysis}\label{sec:method:complexity}
We compare \rankmixer and \unimixer via the computational and parameter complexity of token mixing and P-FFN modules.

\subsubsection*{\bf On token mixing.} The token mixing module in \rankmixer is implemented as a block-transpose permutation, which incurs $\mathcal{O}(TD)$ computation and introduces no additional parameters. 
In contrast, \unimixer employs \textbf{Parameterized Full Mixing}, which applies a dense linear transformation over the flattened representation. 
This results in $\mathcal{O}(T^2D^2)$ computation and $\mathcal{O}(T^2D^2)$ parameters.

\subsubsection*{\bf On P-FFN} The P-FFN module in \rankmixer has computational complexity $\mathcal{O}(TrD^2)$ and parameter complexity $\mathcal{O}(rD^2)$. 
The \unimixer P-FFN adopts a \textbf{GLU-improved P-FFN}, which increases the feed-forward cost to $\mathcal{O}\big(T(3rD^2 + D^2)\big)$ with parameter complexity $(3r+1)D^2$. 
This introduces a constant-factor increase compared with the P-FFN in \rankmixer.

Overall, the primary complexity difference between \rankmixer and \unimixer arises from the token-mixing module, where \unimixer introduces additional parameters and computation through Parameterized Full Mixing. 
The GLU-improved P-FFN contributes a smaller, constant-factor increase in complexity.



\section{Experiments}
\label{sec:exps}

In this section, we conduct extensive benchmark experiments to answer the following research questions:
\begin{itemize}[leftmargin=15pt,parsep=2pt,itemsep=2pt,topsep=2pt]
    \item {\bf RQ1:} Does \unimixer outperform existing baselines on standard CTR prediction benchmarks?
    \item {\bf RQ2:} How effectively does \unimixer mitigate embedding collapse compared with \rankmixer?
    \item {\bf RQ3:} How does \unimixer scale compared with \rankmixer?
    \item {\bf RQ4:} To what extent does \unimixer generalize beyond CTR prediction tasks?
\end{itemize}


\begin{table}[t]
    \centering
    \caption{Statistics of benchmark datasets.}
    \begin{tabular}{lcccc}
        \toprule
        Dataset & \#Train Size & \#Valid Size & \#Test Size & \#Fields \\
        \midrule
        Criteo   & 33.0M & 8.3M & 4.6M & 39 \\
        Avazu  & 32.3M & 4.0M & 4.0M & 24 \\
        \bottomrule
    \end{tabular}
    \label{tab:stat}
\end{table}

\subsection{Setup}\label{sec:exps:setup}
\paragraph{Datasets.} We conduct experiments on two industrial-scale CTR prediction benchmarks, Criteo~\cite{criteo} and Avazu~\cite{avazu}, which are widely used for evaluating recommendation models in practical settings. 
Dataset statistics are summarized in Table~\ref{tab:stat}.

\paragraph{Metrics.} We report AUC and LogLoss, following standard evaluation protocols for CTR prediction~\cite{embedding-collapse-1-multiembedding,feagen,dcnv2+crossnet,xdeepfm+cin,autoint}. 
To evaluate representation collapse, we report effective rank defined in Eq.~\ref{eq:norm-erank}.

\paragraph{Baselines.} We compare \unimixer with \rankmixer and several representative feature-interaction models that demonstrate strong performance in real-world CTR systems, including xDeepFM~\cite{xdeepfm+cin}, DCNv2~\cite{dcnv2+crossnet}, AutoInt~\cite{autoint}, and a standard DNN (MLP). 
For both \unimixer and \rankmixer, we use a two-layer architecture (i.e., two stacked blocks). 
We set \((T, D) = (15, 26)\) on Criteo and \((T, D) = (16, 24)\) on Avazu. 
The hidden dimension of the P-FFNs is set to \(D\), and the GLU-improved P-FFNs in \unimixer use an expansion ratio \(r = 3\). 
For the remaining baselines, we adopt the recommended configurations provided in the open-source FuxiCTR library~\cite{benchmark-fuxictr}.

\paragraph{Experimental protocol.} All models are implemented using the open-source FuxiCTR benchmark library~\cite{benchmark-fuxictr}, following its officially recommended training pipeline. 
We set the embedding dimension to 20 for Criteo and 16 for Avazu. 
Models are trained for up to 100 epochs with early stopping triggered when validation loss does not improve for two consecutive epochs. 
We use batch sizes of 4096 for Criteo and 10000 for Avazu. 
Each reported result is averaged over 10 runs with different random initializations to ensure statistical reliability.



\begin{table}[t]
\centering
\caption{Overall CTR prediction performance of \unimixer and baseline models on Criteo and Avazu.}
\label{tab:overall}
\begin{tabular}{lcc|cc}
\toprule
& \multicolumn{2}{c}{Criteo} & \multicolumn{2}{c}{Avazu} \\
Model & AUC $\uparrow$ & LogLoss $\downarrow$ & AUC $\uparrow$ & LogLoss $\downarrow$ \\
\midrule
MLP  & 0.81307 & 0.43927  & 0.79226 & 0.37247 \\
xDeepFM  & 0.81334 & 0.43849  & 0.79242 & 0.37236 \\
DCNv2  & 0.81365 & 0.43816  & 0.79258 & 0.37227 \\
AutoInt  & 0.81331 & 0.43853  & 0.79072 & 0.37430	\\
\rankmixer  & 0.81375 & 0.43799  & 0.79270 & 0.37218 \\ \midrule
\unimixer  & \textbf{0.81482} & \textbf{0.43730}  & \textbf{0.79323} & \textbf{0.37196} \\
\bottomrule
\end{tabular}
\end{table}

\begin{table}[t]
\centering
\caption{Module-wise ablation study for token mixing and P-FFN in \unimixer and \rankmixer.}
\label{tab:ablation}
\setlength{\tabcolsep}{2pt}
\begin{tabular}{ccc|cc}
\toprule
                 & \multicolumn{2}{c|}{Criteo}             & \multicolumn{2}{c}{Avazu}               \\
Model            & AUC $\uparrow$   & LogLoss $\downarrow$ & AUC $\uparrow$   & LogLoss $\downarrow$ \\ \midrule
\unimixer        & \textbf{0.81482} & \textbf{0.43730}     & \textbf{0.79323} & \textbf{0.37196}     \\
w/o Full Mixing  & 0.81413          & 0.43785              & 0.79289          & 0.37210              \\
ReLU-based FFN  & 0.81326          & 0.43869              & 0.79241          & 0.37229              \\
GELU-based FFN  & 0.81349          & 0.43851              & 0.79288          & 0.37212              \\ \midrule
\rankmixer       & 0.81375          & 0.43799              & 0.79270          & 0.37218              \\
GLU-style FFNon & 0.81393          & 0.43802              & 0.79286          & 0.37212              \\
ReLU-based FFN  & 0.81284          & 0.43917              & 0.79227          & 0.37238              \\ \bottomrule
\end{tabular}
\end{table}

\begin{figure}[t]
  \centering
    \begin{subfigure}{0.495\linewidth}
        \centering
        \includegraphics[width=1.02\linewidth]{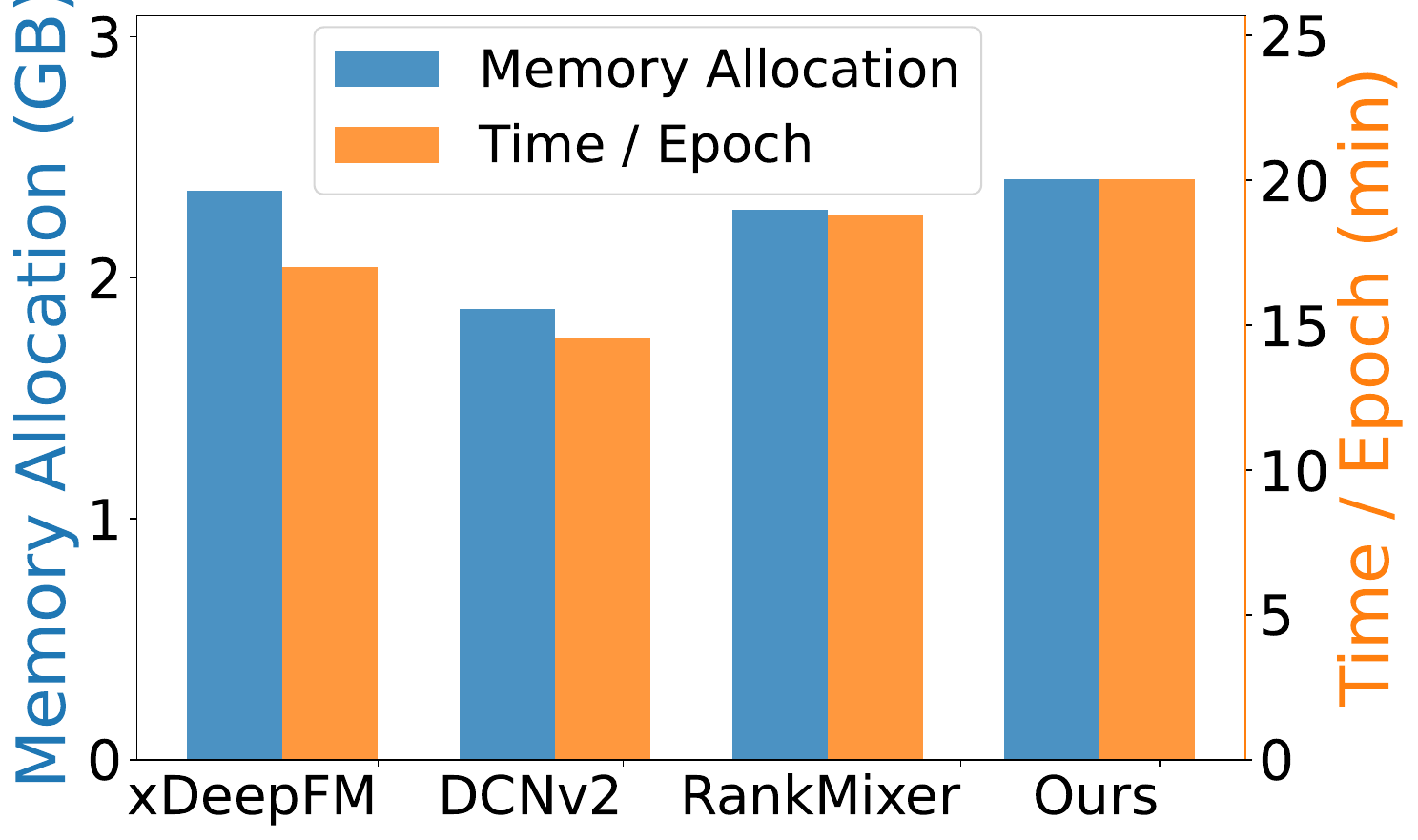}
        \caption{Results on Criteo.}
        \label{fig:efficiency-criteo}
    \end{subfigure}
    \begin{subfigure}{0.495\linewidth}
        \centering
        \includegraphics[width=1.02\linewidth]{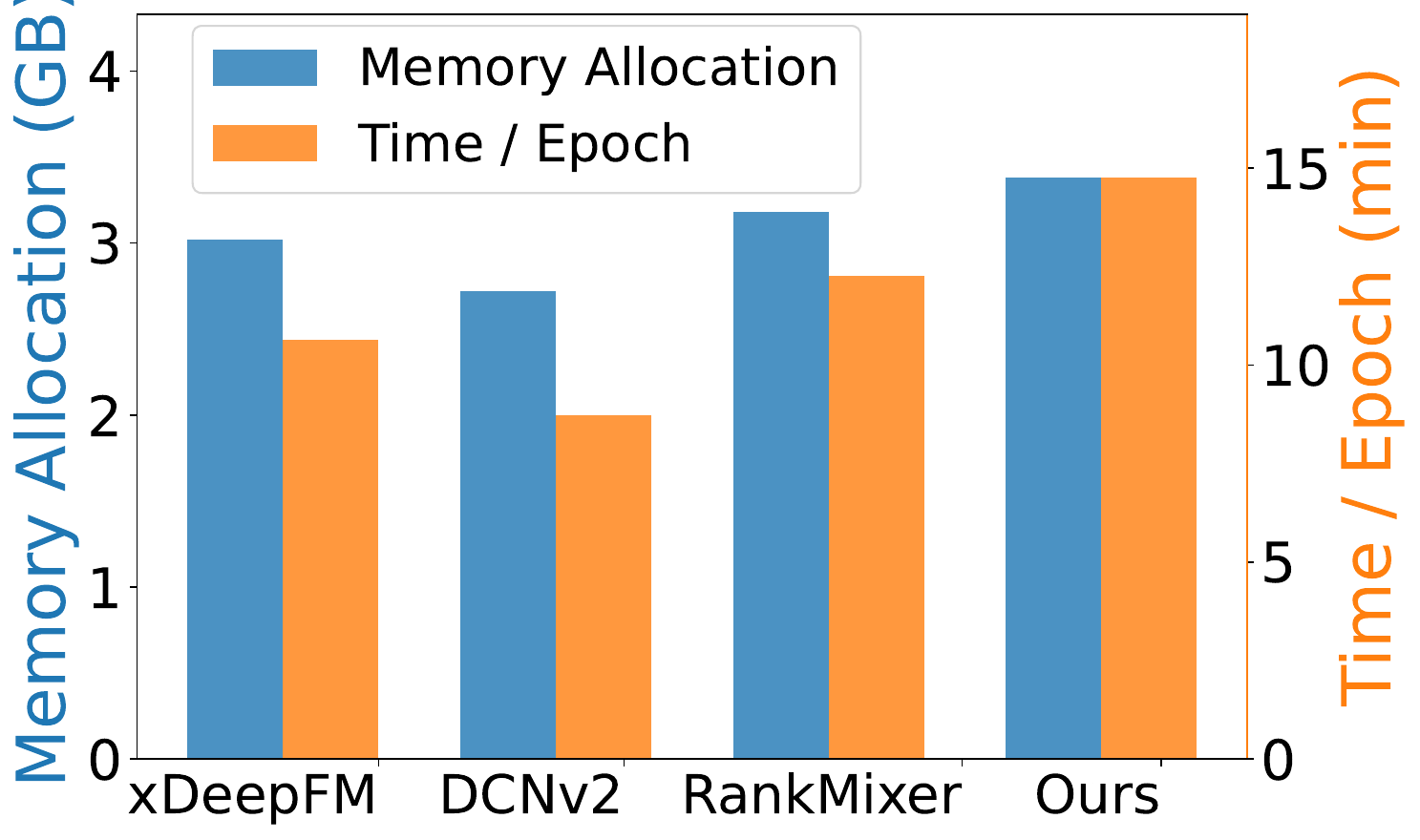}
        \caption{Results on Avazu.}
        \label{fig:efficiency-avazu}
    \end{subfigure}
    %
  \caption{Efficiency comparison of \unimixer and \rankmixer, alongside other baselines for reference.}
  \label{fig:efficiency}
\end{figure}

\begin{figure*}[!t]
  \centering
    \begin{subfigure}{0.247\textwidth}
        \centering
        \includegraphics[width=\linewidth]{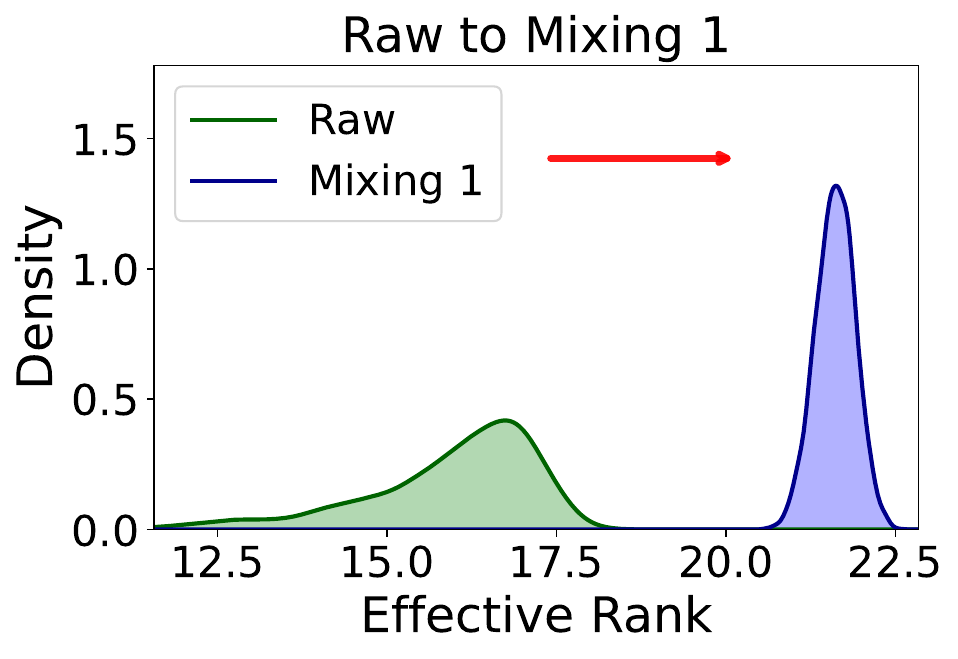}
        \label{fig:erank-criteo-1-main}
    \end{subfigure}
    \begin{subfigure}{0.247\textwidth}
        \centering
        \includegraphics[width=\linewidth]{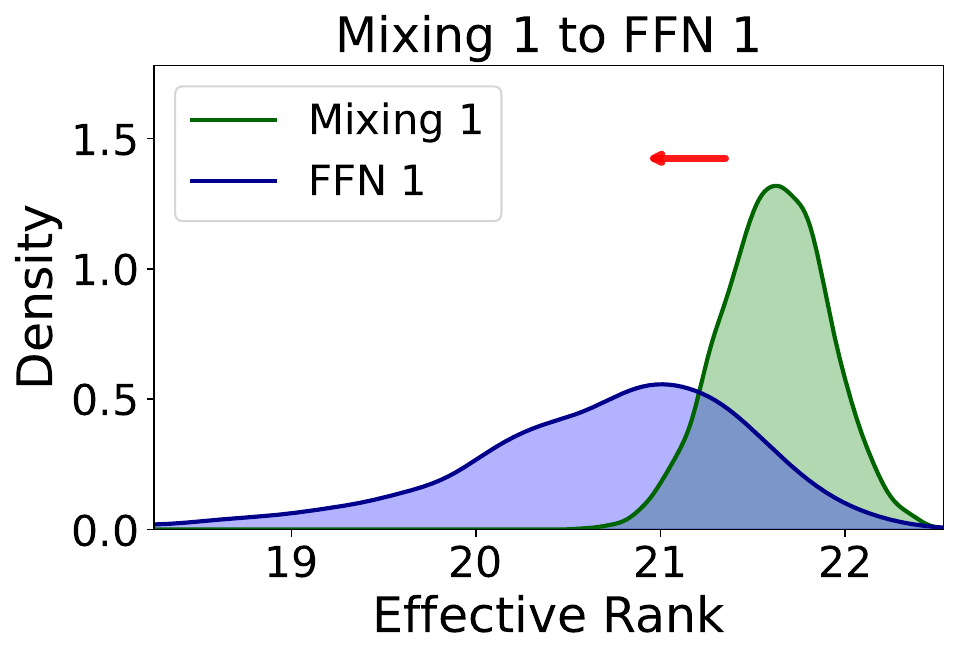}
        \label{fig:erank-criteo-2-main}
    \end{subfigure}
    \begin{subfigure}{0.247\textwidth}
        \centering
        \includegraphics[width=\linewidth]{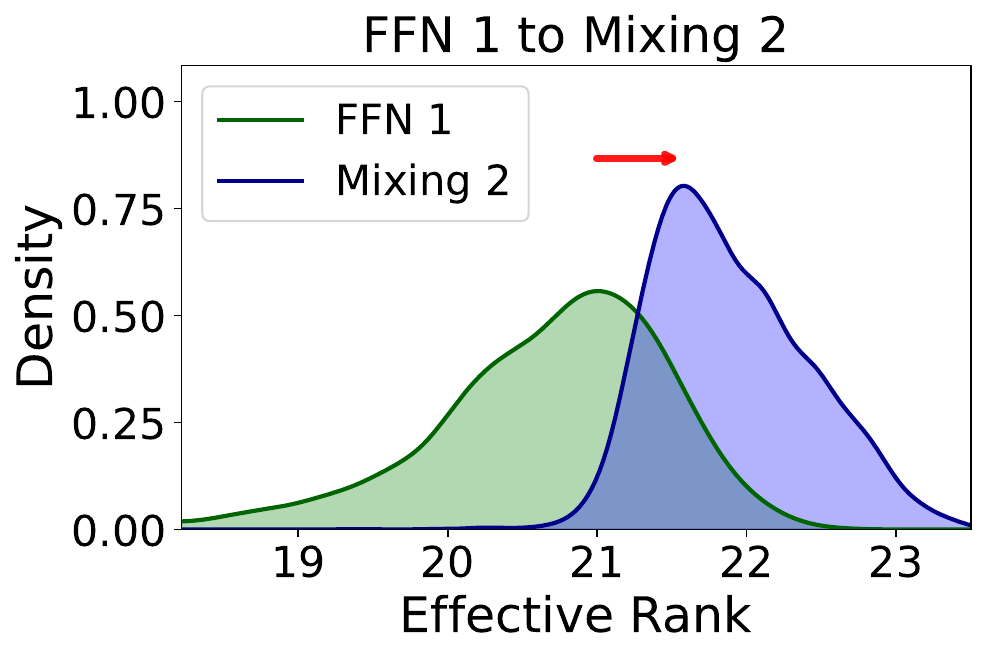}
        \label{fig:erank-criteo-3-main}
    \end{subfigure}
    \begin{subfigure}{0.247\textwidth}
        \centering
        \includegraphics[width=\linewidth]{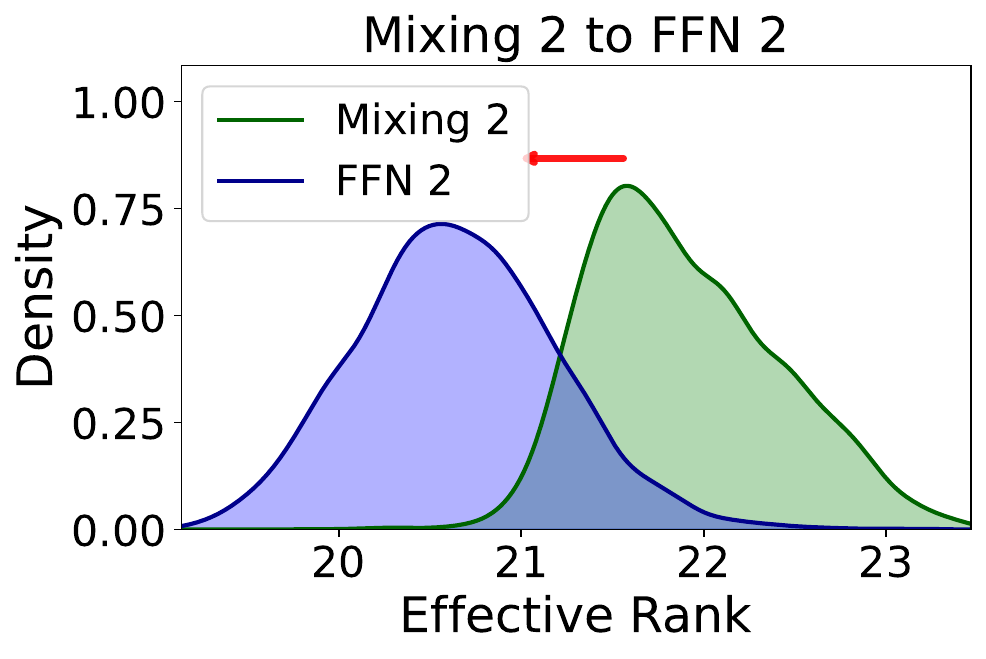}
        \label{fig:erank-criteo-4-main}
    \end{subfigure}
    \\
    \begin{subfigure}{0.247\textwidth}
        \centering
        \includegraphics[width=\linewidth]{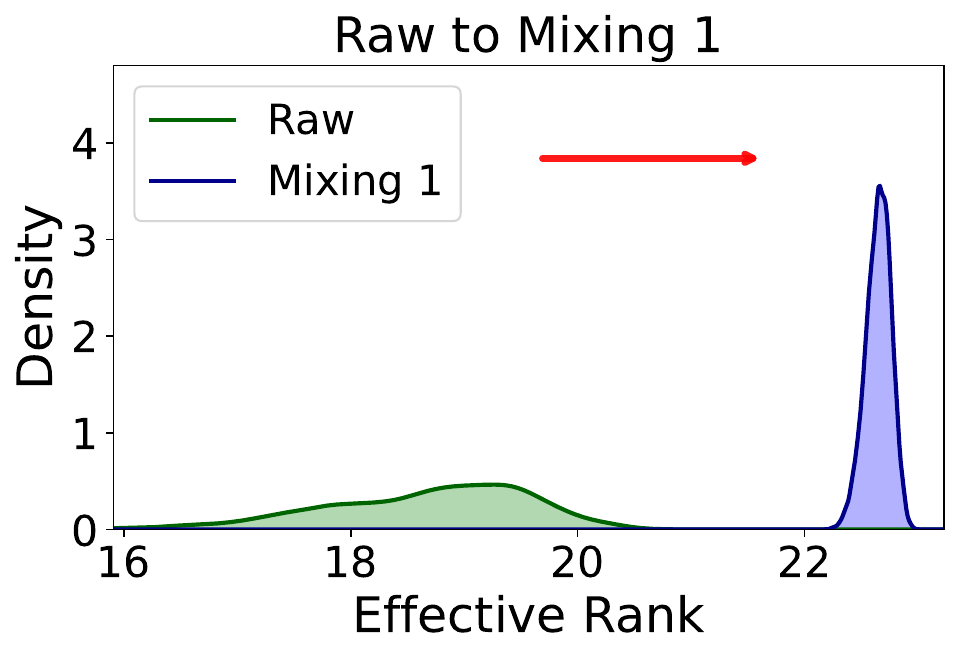}
        \caption{Raw Embedding $\to$ Mixing 1}
        \label{fig:erank-avazu-1-main}
    \end{subfigure}
    \begin{subfigure}{0.247\textwidth}
        \centering
        \includegraphics[width=\linewidth]{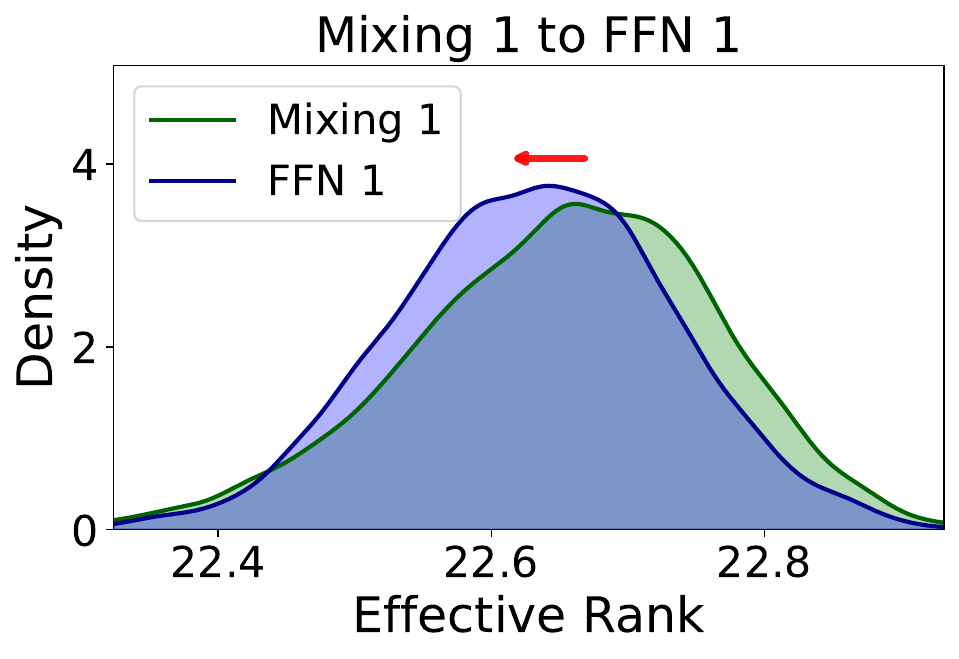}
        \caption{Mixing 1 $\to$ FFN 1}
        \label{fig:erank-avazu-2-main}
    \end{subfigure}
    \begin{subfigure}{0.247\textwidth}
        \centering
        \includegraphics[width=\linewidth]{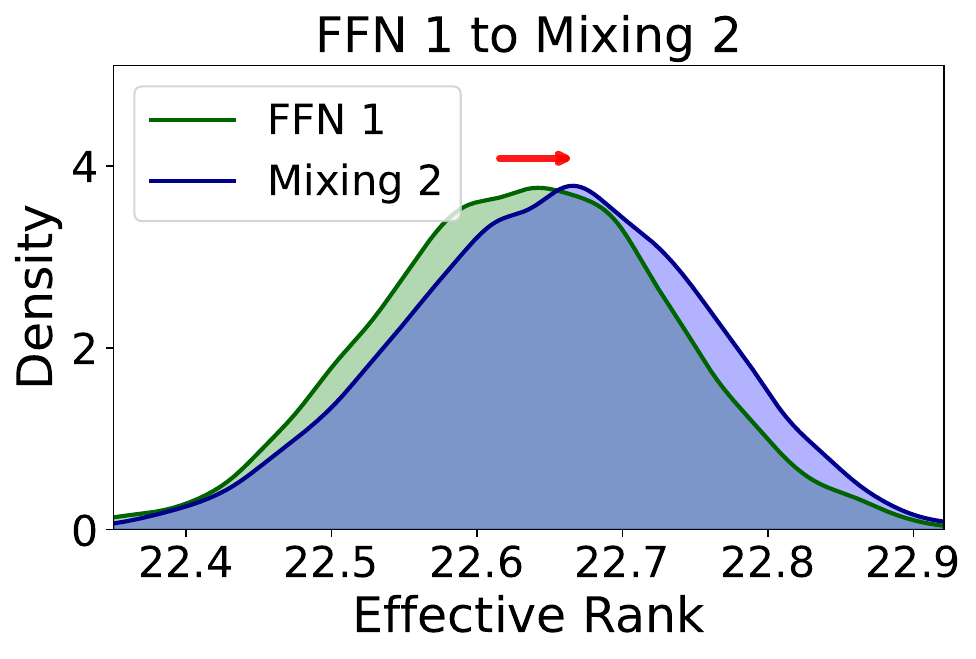}
        \caption{FFN 1 $\to$ Mixing 2}
        \label{fig:erank-avazu-3-main}
    \end{subfigure}
    \begin{subfigure}{0.247\textwidth}
        \centering
        \includegraphics[width=\linewidth]{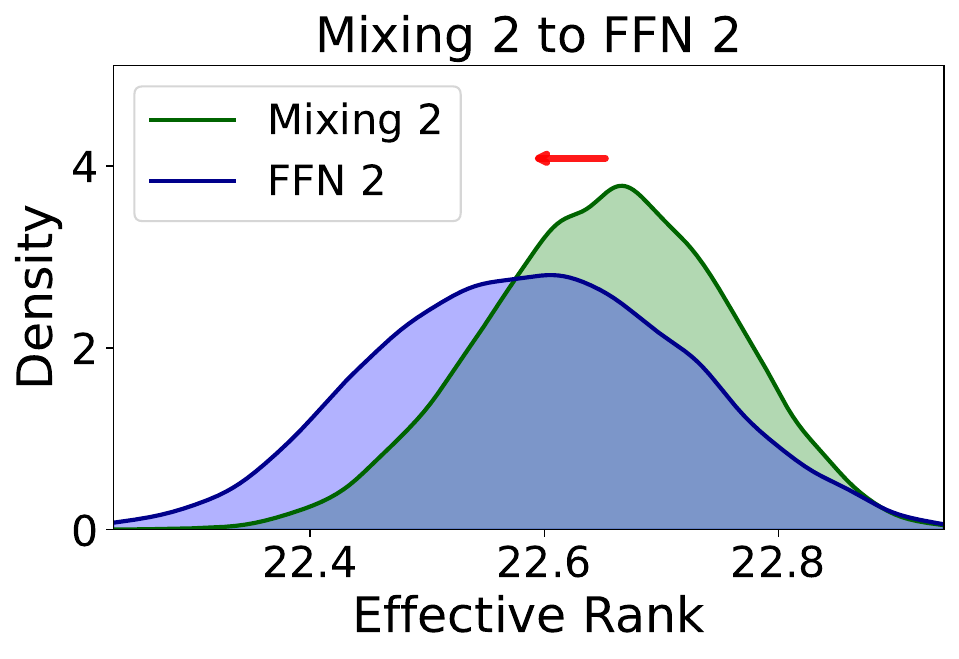}
        \caption{Mixing 2 $\to$ FFN 2}
        \label{fig:erank-avazu-4-main}
    \end{subfigure}
  \caption{Distributional shift of per-sample effective rank across representation stages in \unimixer, from raw embeddings to successive modules on Criteo (top) and Avazu (bottom). Results correspond to the illustration in Figure~\ref{fig:erank-rankmixer}.}
  \label{fig:erank-main}
\end{figure*}

\subsection{RQ1: CTR Prediction Performance}
\label{sec:exps:exps_overall}

\subsubsection*{\bf Quantitative comparisons.} We first present the quantitative performance comparison between \unimixer and baseline models. 
As shown in Table~\ref{tab:overall}, \unimixer consistently achieves better performance than all baselines on both AUC and LogLoss, including \rankmixer. In particular, \unimixer improves AUC by up to \textbf{0.001} over the strongest competing model, which is considered a statistically meaningful improvement on industrial-scale CTR benchmarks~\cite{DIN}. 
These results validate the effectiveness of the proposed refinements and the resulting \unimixer architecture. 
Notably, the parameter size of \unimixer remains comparable to other baselines, indicating that the performance gain is achieved without introducing excessive model complexity.

\begin{figure}[!t]
  \centering
    \begin{subfigure}{0.495\linewidth}
        \centering
        \includegraphics[width=1.02\linewidth]{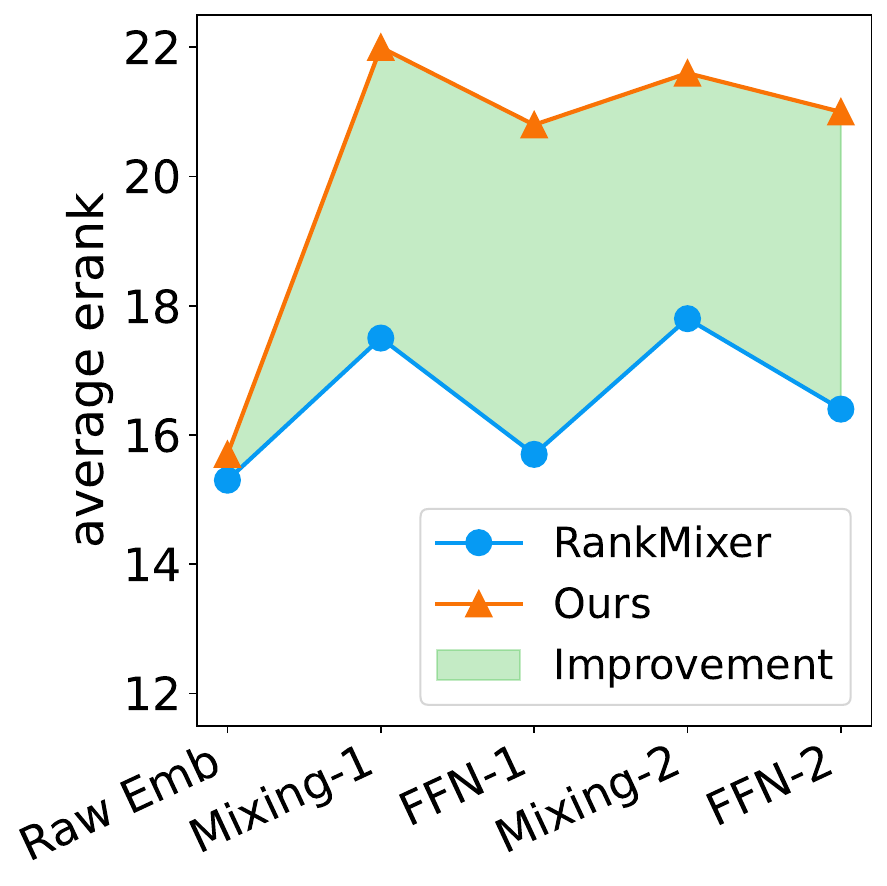}
        \caption{Comparison on Criteo.}
        \label{fig:avg-erank-criteo}
    \end{subfigure}
    \begin{subfigure}{0.495\linewidth}
        \centering
        \includegraphics[width=1.02\linewidth]{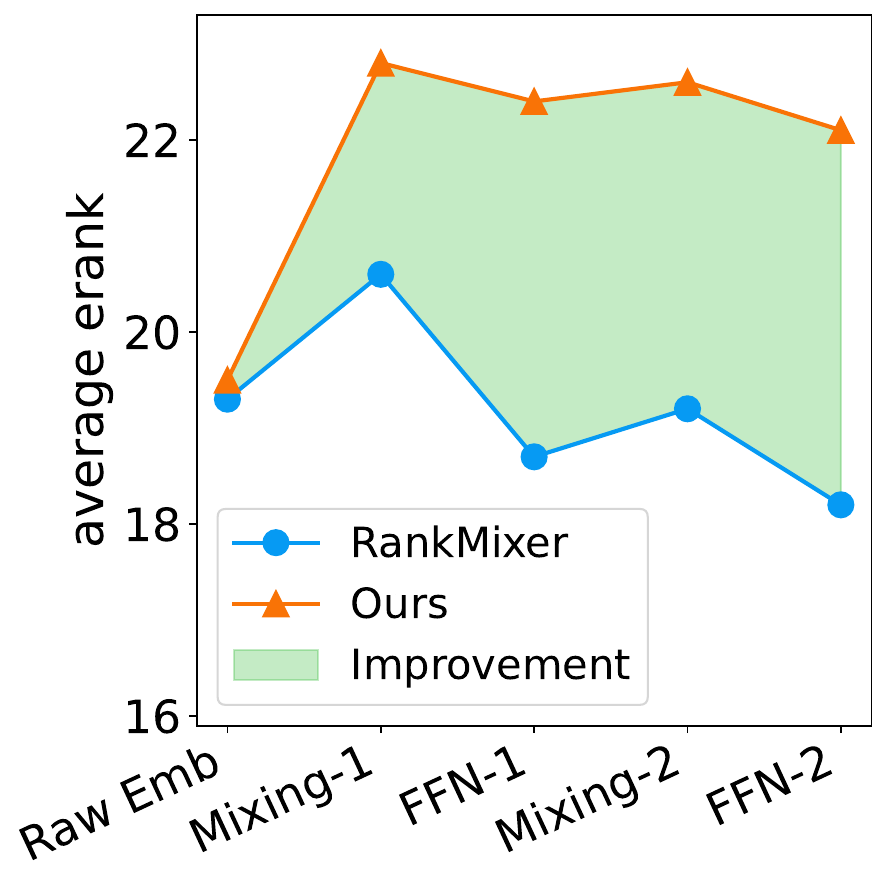}
        \caption{Comparison on Avazu.}
        \label{fig:avg-erank-avazu}
    \end{subfigure}
    %
  \caption{Layer-wise comparison of average effective rank, from raw embeddings to module outputs, over test samples for \unimixer and \rankmixer. Results align with Figure~\ref{fig:avg-erank-rankmixer}.}
  \label{fig:avg-erank}
\end{figure}

\subsubsection*{\bf Module contribution analysis.} We further conduct ablation studies to evaluate the contribution of the proposed modules in \unimixer, together with corresponding ablations applied to \rankmixer. 
As shown in Table~\ref{tab:ablation}, both Parameterized Full Mixing and the GLU-improved P-FFN contribute significantly to the performance of \unimixer, and removing either component leads to consistent performance degradation. 
Moreover, the two modules exhibit a clear collaborative effect: combining them yields substantially larger improvements than using either module alone. 
A similar modification applied to \rankmixer (e.g., replacing the FFN with a GLU-style activation) results in only marginal performance gains, suggesting that the benefits of the GLU-improved P-FFN depend on the enhanced token-mixing capability introduced by parameterized full mixing. 
These observations are consistent with our theoretical analysis, which highlights the complementary roles of the two modules in mitigating embedding collapse.

\begin{figure*}[!t]
  \centering
    \begin{subfigure}{0.495\textwidth}
        \centering
        \includegraphics[width=\linewidth]{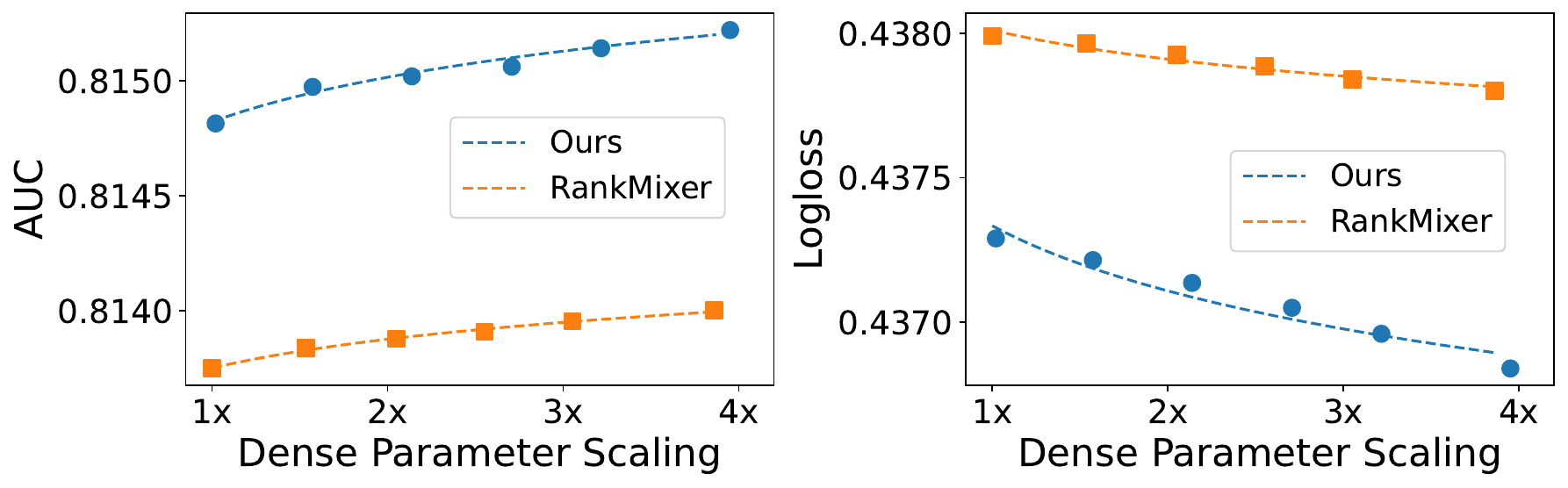}
        \caption{Width-wise dense parameter scaling on Criteo.}
        \label{fig:w-scaling-law-criteo}
    \end{subfigure}
    \begin{subfigure}{0.495\textwidth}
        \centering
        \includegraphics[width=\linewidth]{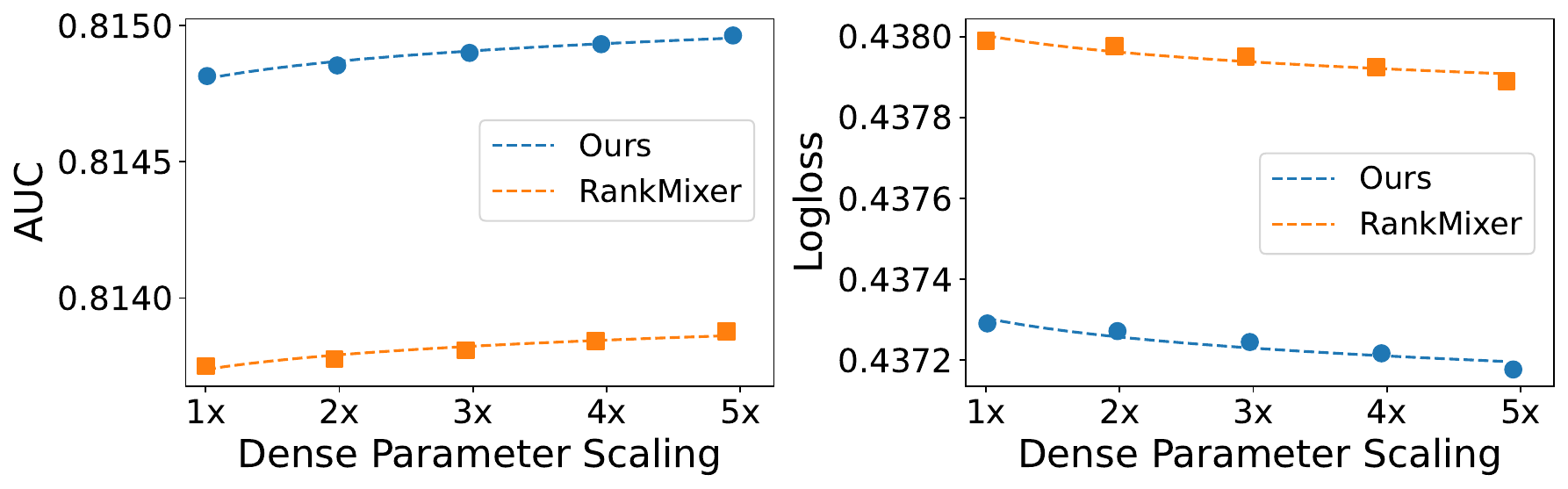}
        \caption{Depth-wise dense parameter scaling on Criteo.}
        \label{fig:d-scaling-law-criteo}
    \end{subfigure}
    \\ 
    \begin{subfigure}{0.495\textwidth}
        \centering
        \includegraphics[width=\linewidth]{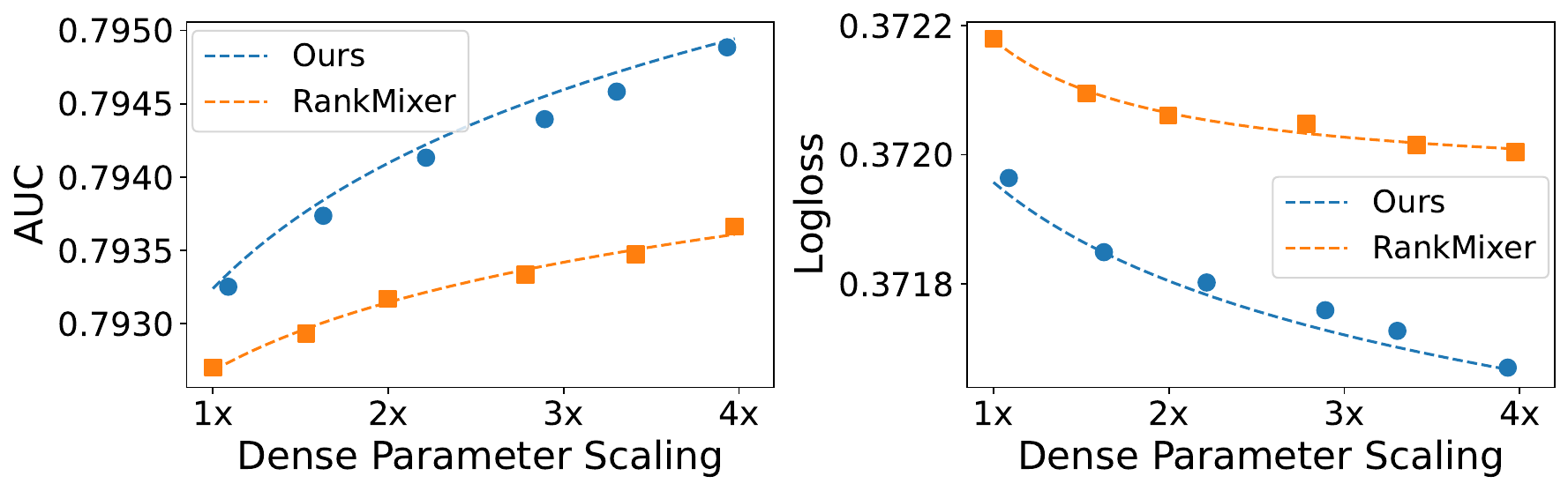}
        \caption{Width-wise dense parameter scaling on Avazu.}
        \label{fig:w-scaling-law-avazu}
    \end{subfigure}
    \begin{subfigure}{0.495\textwidth}
        \centering
        \includegraphics[width=\linewidth]{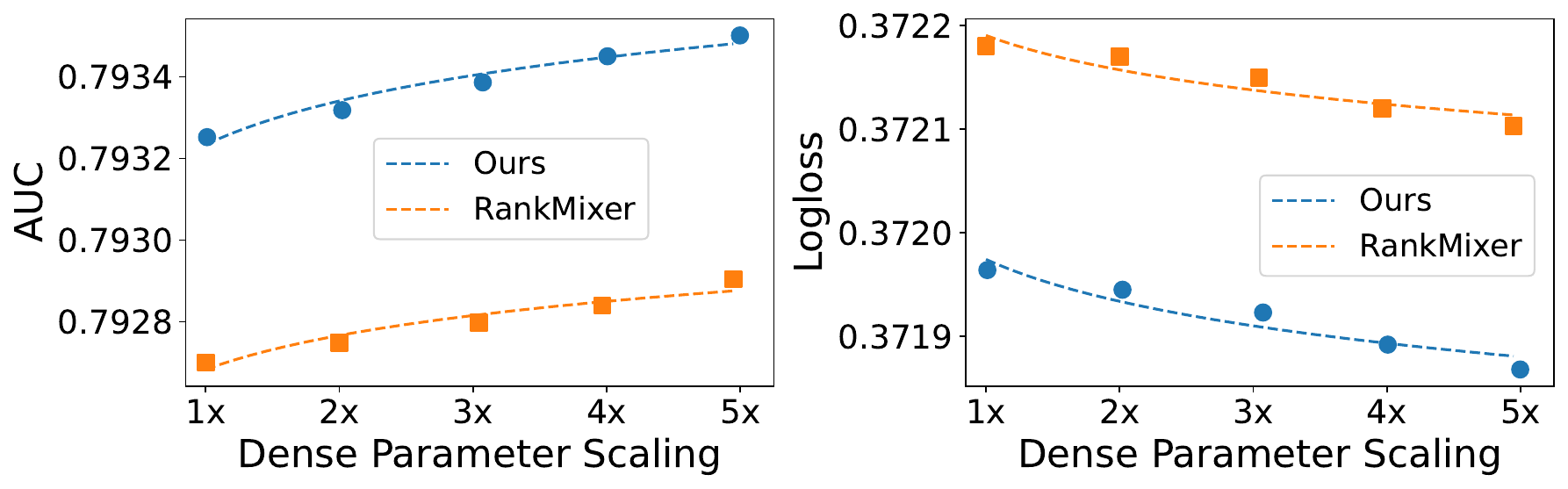}
        \caption{Depth-wise dense parameter scaling on Avazu.}
        \label{fig:d-scaling-law-avazu}
    \end{subfigure}
    %
  \caption{Dense parameter scaling trends under width (left) and depth (right) scaling for \rankmixer and \unimixer. The x-axis shows the parameter scaling factor relative to the base model. Markers denote empirical measurements, while dashed lines indicate fitted scaling curves capturing the performance–parameter scaling relationship.}
  \label{fig:scaling-appendix}
\end{figure*}

\subsubsection*{\bf Efficiency comparison.} We evaluate the efficiency of \unimixer against \rankmixer in terms of \textit{training time per epoch} and \textit{GPU memory usage}, with results shown in Figure~\ref{fig:efficiency}. 
\unimixer introduces a modest 10\%–15\% increase in runtime compared to \rankmixer, while maintaining similar GPU memory requirements and remaining competitive with other efficient baselines such as DCNv2. 
This aligns with the complexity analysis presented in Section~\ref{sec:method:complexity}, indicating that the additional computational cost is practically negligible for large-scale recommendation. 
Overall, \unimixer achieves a favorable tradeoff, offering superior recommendation performance with only minor efficiency overhead.


\subsection{RQ2: Embedding Collapse Mitigation}
\label{sec:exps:exps_collapse}
We now examine how the effective rank evolves in \unimixer in comparison with \rankmixer. 
Specifically, we analyze both \textbf{(\romannumeral1)} the \textit{effective-rank distribution shift} and \textbf{(\romannumeral2)} the \textit{average effective rank across representation stages}, providing a comprehensive view of representation diversity throughout the models.

\subsubsection*{\bf Effective-rank distribution shift.} As shown in Figure~\ref{fig:erank-main}, we visualize per-sample effective-rank distributions from raw embeddings to successive module outputs on Criteo (top) and Avazu (bottom). 
Compared with the corresponding results of \rankmixer in Figure~\ref{fig:erank-rankmixer}, \unimixer exhibits substantially larger effective-rank distribution shifts across internal representations. 
While both models start with nearly identical effective-rank distributions at the raw embedding stage, their behaviors diverge as representations propagate through the network. 
In particular, the parameterized full mixing in \unimixer produces noticeably larger effective-rank gains than the block-transposed mixing in \rankmixer, and the GLU-enhanced P-FFN maintains stronger rank preservation compared with the P-FFN used in \rankmixer.

\subsubsection*{\bf ``Expand more, shrink less''.} Figure~\ref{fig:avg-erank} shows the layer-wise average effective rank for both models, corresponding to Figure~\ref{fig:avg-erank-rankmixer}. 
We observe that the spectral dynamics of \unimixer follow the same characteristic alternating behavior previously identified in \rankmixer: token mixing layers tend to expand the representation spectrum and increase effective rank, whereas P-FFN layers introduce rank contraction. 
However, compared with \rankmixer, \unimixer consistently exhibits noticeably stronger expansion effects together with substantially milder contraction after each P-FFN stage. 
This behavior leads to more stable spectral evolution across layers, allowing \unimixer to consistently maintain higher average effective rank than \rankmixer across datasets and representation stages. 
Notably, on Avazu — where \rankmixer exhibits clear collapse tendencies — \unimixer preserves steady effective-rank growth throughout the network, indicating improved robustness against representation collapse.

Overall, these results confirm that \unimixer mitigates collapse more effectively than \rankmixer, producing more stable and expressive internal representations across layers.


\subsection{RQ3: Dense Parameter Scaling Analysis}
\label{sec:exps:exps_scaling}
To evaluate the scaling behavior of \unimixer as a scalable deep recommender paradigm, we conduct scaling-law experiments on its dense parameters, including those in Parameterized Full Mixing and GLU-improved P-FFNs. 
We vary \textbf{(\romannumeral1)} \textit{depth} (number of blocks) and \textbf{(\romannumeral2)} \textit{width} (hidden size of FFNs), and examine both individual and joint scaling schemes, with comparisons to \rankmixer.

Figure~\ref{fig:scaling-appendix} shows the results of width-wise (left) and depth-wise (right) scaling, while Figure~\ref{fig:scaling} presents the joint depth and width scaling results. 
Across all configurations, \unimixer exhibits substantially better scaling behavior than \rankmixer, with larger AUC improvements and lower LogLoss as model parameters increase. 
These results indicate that \unimixer can serve as a more effective paradigm for deep recommender systems, benefiting from its theoretically grounded design that mitigates embedding collapse while enhancing downstream performance.

Furthermore, comparing joint scaling (Figure~\ref{fig:scaling}) with individual depth- or width-scaling (Figure~\ref{fig:scaling-appendix}), we observe that increasing both depth and width consistently yields larger performance gains than scaling a single dimension. 
This behavior aligns with findings in language models~\cite{scaling-law} and prior observations in \rankmixer~\cite{rankmixer}, demonstrating the consistent but more significant scaling properties of \unimixer.

\begin{figure}[!t]
  \centering
    \begin{subfigure}{\linewidth}
        \centering
        \includegraphics[width=1.01\linewidth]{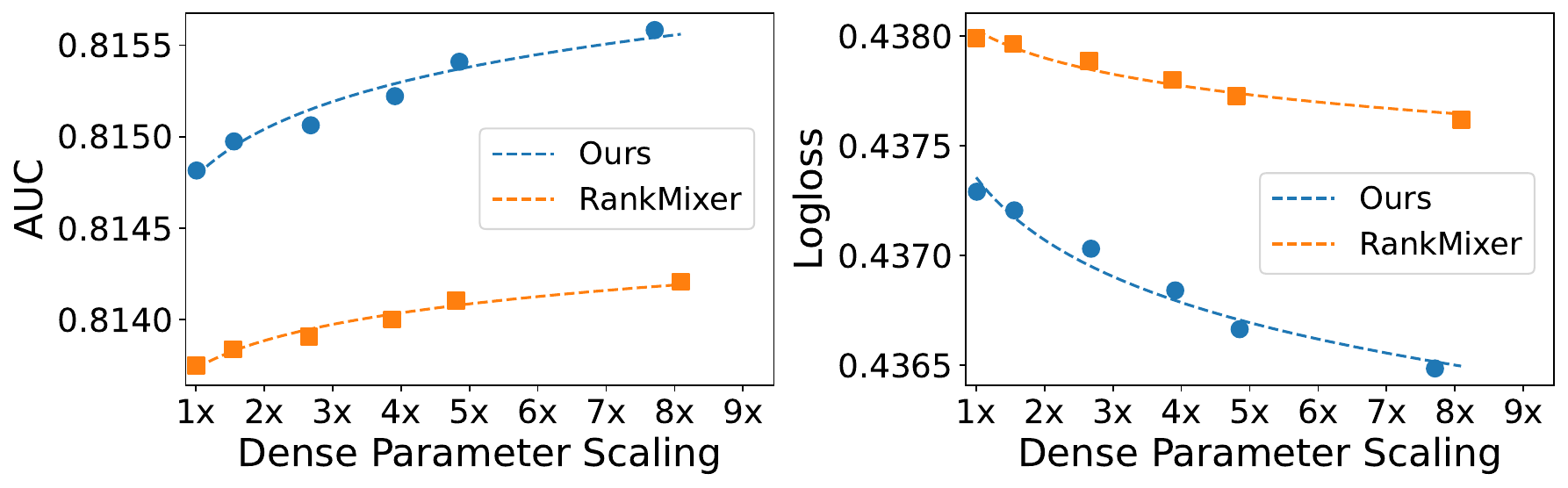}
        \caption{Joint width and depth dense parameter scaling on Criteo.}
        \label{fig:scaling-law-criteo}
    \end{subfigure}
    \\ 
    \begin{subfigure}{\linewidth}
        \centering
        \includegraphics[width=1.01\linewidth]{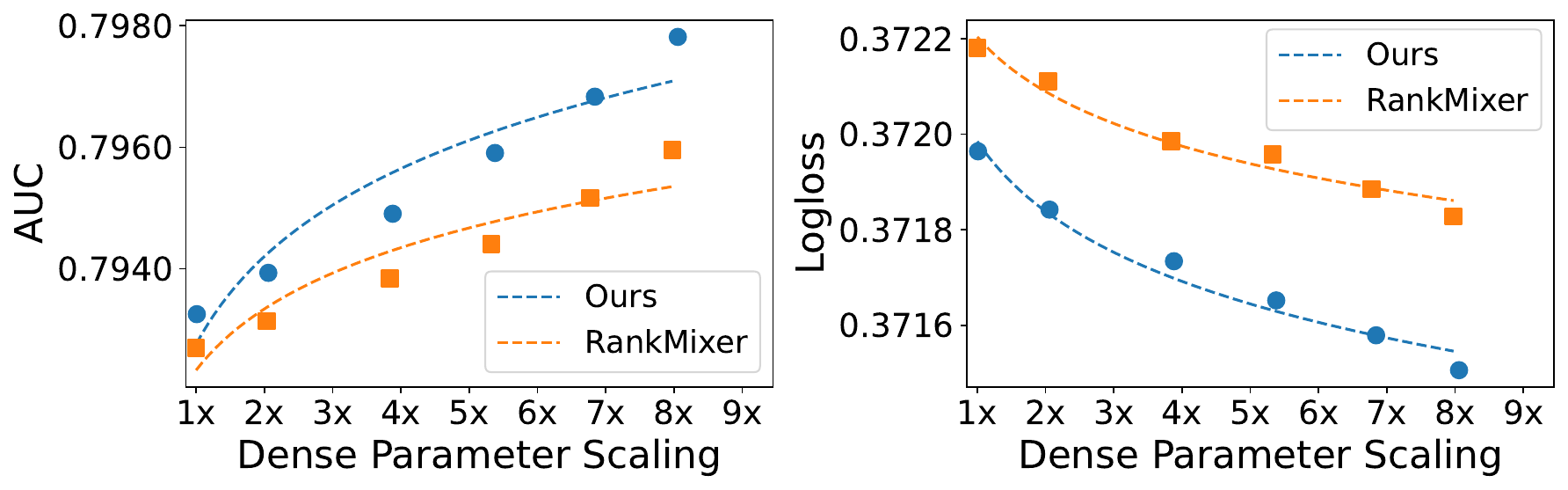}
        \caption{Joint depth and width dense parameter scaling on Avazu.}
        \label{fig:scaling-law-avazu}
    \end{subfigure}
    %
  \caption{Dense parameter scaling trends under joint width and depth scaling for \rankmixer and \unimixer.}
  \label{fig:scaling}
\end{figure}


\subsection{RQ4: Behavior Sequence Modeling Generalization}
\label{sec:exps:more-exp}

To further evaluate the generalization capability of \unimixer beyond standard CTR prediction, we additionally assess our method on \textit{behavior sequence modeling} tasks. 
Specifically, we conduct experiments on two datasets from the FuxiCTR benchmark: KuaiVideo (K), which predicts users' click probabilities on short videos, and TaobaoAd (T), which models user shopping behaviors. 
Following the benchmark protocol, we report both AUC and Group-AUC (gAUC), and compare \unimixer against several strong baselines, including \rankmixer.

As shown in Table~\ref{tab:more-exp}, \unimixer consistently achieves the best performance across both datasets and evaluation metrics, demonstrating clear advantages over all competing methods. 
Moreover, we observe similar spectral dynamics improvements to those shown in Figure~\ref{fig:avg-erank}, suggesting that both our analytical framework and the proposed unified mixing design generalize effectively beyond standard CTR estimation tasks. 
These results highlight the potential of our approach to broader recommendation scenarios.

\begin{table}[!t]
\centering
\caption{Comparison between \unimixer and \rankmixer on behavior sequence modeling tasks. (K) and (T) denote results on KuaiVideo and TaobaoAd datasets, respectively.}
\label{tab:more-exp}
\resizebox{\linewidth}{!}{
\begin{tabular}{ccc|cc}
\toprule
Model     & gAUC (K) & AUC (K)  & gAUC (T) & AUC (T)  \\ \midrule
AutoInt   & 0.6667 & 0.7469 & 0.5744 & 0.6486 \\
DCNv2     & 0.6675 & 0.7470 & 0.5749 & 0.6495 \\
xDeepFM     & 0.6696 & 0.7471 & 0.5729 & 0.6393 \\
\rankmixer & 0.6691 & 0.7482 & 0.5763 & 0.6508 \\ \midrule
\unimixer      & \textbf{0.6731} & \textbf{0.7514} & \textbf{0.5778} & \textbf{0.6522} \\
\bottomrule
\end{tabular}}
\end{table}
\section{Conclusion}
\label{sec:conclusion}

\subsubsection*{\bf Paper summary.} In this paper, we analyzed \rankmixer from the perspective of embedding collapse and showed that, although it provides modest improvements in effective rank compared to conventional recommenders, these gains are typically insufficient to prevent collapse, highlighting limitations in its core token-mixing and per-token FFN modules. 
Motivated by these insights, we proposed \unimixer, a novel deep recommender that enhances \rankmixer with (\romannumeral1) Parameterized Full Mixing and (\romannumeral2) GLU-improved per-token FFNs, which provably enable more expressive representation transformations and produce spectrum-robust token embeddings. 
Extensive experiments on industrial-scale benchmarks demonstrate that \unimixer consistently achieves superior performance over strong baselines, including \rankmixer, while maintaining higher effective rank and more expressive, non-collapsed representations. 
Moreover, scaling-law analyses confirm that \unimixer exhibits better scaling behavior, highlighting its potential as a robust and scalable paradigm for large-scale recommendation systems.

\subsubsection*{\bf Limitations and future work.} The effective-rank dynamics identified in this work — particularly the damped oscillatory trajectory observed in token-transformation-based recommenders such as \rankmixer and \unimixer — are inherently architecture-dependent, and our theoretical analysis focuses on this class of models. 
Nevertheless, the strong empirical performance of \rankmixer-style architectures suggests substantial room for further exploration. 
Our study provides a spectral perspective on representation collapse in deep recommenders and points to promising directions for future work, including more expressive token-mixing mechanisms, spectrum-preserving nonlinear modules, and deeper investigation of spectral dynamics in this architecture.

\section*{Acknowledgments}

This work was supported by the Tencent Rhino-Bird Focused Research Project under Grant No. P00660.
We thank the anonymous reviewers for their valuable feedback and constructive suggestions. 
We also appreciate the open-source community and benchmark maintainers whose resources supported this research.

\bibliographystyle{ACM-Reference-Format}
\bibliography{references}

\appendix
\section{Justification of Theorem~\ref{thm:erank-bound-blocktranspose}}
\label{app:thm:erank-bound-blocktranspose}

In this section we justify Theorem~\ref{thm:erank-bound-blocktranspose}. For completeness, we restate the setup and result before presenting the derivations.

\subsubsection*{\bf Restated theorem.} 
Let $X\in\mathbb R^{T\times D}$ with $D=Td$ and $d\ge 1$. Partition $X$ into a $T\times T$ grid of row-blocks $X_{ij}\in\mathbb R^{1\times d}$. Define the block-transpose operator $\mathcal T$ by
\[
(\mathcal T(X))_{ij}=X_{ji},
\]
and let $Y=\mathcal T(X)$.
Denote
\[
\operatorname{rank}(X)=r,
\qquad
\operatorname{erank}(X)=k,
\qquad
\mu=\operatorname{erank}(Y).
\]
Assume Frobenius-orthogonality and spectral incoherence:
\[
\langle X,Y\rangle_F \approx 0,
\qquad
\|X+Y\|_2^2 \approx \max\{\|X\|_2^2,\|Y\|_2^2\}.
\]
Let $M=X+Y$. The theorem states
\[
\frac{2k\mu}{(\sqrt{k}+\sqrt{\mu})^2}
\le
\operatorname{erank}(M)
\le
2(k+\mu),
\qquad
\mu \le \operatorname{rank}(Y)\le \min\{T,rd\}.
\]

We prove the result in two steps: first bounding the algebraic rank of $Y=\mathcal T(X)$, and then analyzing the effective rank of $M=X+\mathcal T(X)$.

\subsection{Rank properties of the block-transpose operator}

We begin with a rank characterization of the block-transpose mapping.

\subsubsection*{\bf Setup.}
Write a rank-$r$ decomposition
\[
X=\sum_{\ell=1}^r \sigma_\ell\,\mathbf u_\ell \mathbf v_\ell^\top,
\]
where $\mathbf u_\ell\in\mathbb R^T$ and $\mathbf v_\ell\in\mathbb R^{Td}$. Partition each $\mathbf v_\ell$ into $T$ contiguous segments $\mathbf v_\ell^{(j)}\in\mathbb R^{1\times d}$.

\subsubsection*{\bf Rank-1 case.}
Let $X=\mathbf u\mathbf v^\top$. Then $X_{ij}=u_i\,\mathbf v^{(j)}$ and
\[
Y_{ij}=X_{ji}=u_j\,\mathbf v^{(i)}.
\]
For the $j$-th block-column of $Y$,
\[
C_j
=
\begin{bmatrix}
Y_{1j}\\ \vdots\\ Y_{Tj}
\end{bmatrix}
=
u_j
\begin{bmatrix}
\mathbf v^{(1)}\\ \vdots\\ \mathbf v^{(T)}
\end{bmatrix}.
\]
Define
\[
V_{\mathrm{stack}}
=
\begin{bmatrix}
\mathbf v^{(1)}\\ \vdots\\ \mathbf v^{(T)}
\end{bmatrix}
\in\mathbb R^{T\times d}.
\]
Then $C_j=u_j V_{\mathrm{stack}}$, so all block-columns of $Y$ lie in the column space of $V_{\mathrm{stack}}$. Hence
\[
\operatorname{rank}(Y)
=
\operatorname{rank}(V_{\mathrm{stack}})
\le \min\{T,d\}.
\]

\subsubsection*{\bf General rank-$r$ case.}
By linearity of $\mathcal T$,
\[
Y=\sum_{\ell=1}^r
\sigma_\ell\,\mathcal T(\mathbf u_\ell\mathbf v_\ell^\top).
\]
Each term contributes at most $d$ independent directions in $\mathbb R^T$, so the column space of $Y$ is contained in a subspace of dimension at most $rd$. Since $Y\in\mathbb R^{T\times Td}$, this yields
\[
\operatorname{rank}(Y)\le \min\{T,rd\}.
\]

\subsection{Effective rank of the initialized sum}

We now analyze the effective rank of
\[
M=X+\mathcal T(X)=X+Y.
\]

\subsubsection*{\bf Spectral incoherence assumption.}
We assume the initialization produces approximate Frobenius orthogonality, $\langle X,Y\rangle_F\approx 0$, and incoherent dominant singular directions,
\[
\|X+Y\|_2^2 \approx
\max\{\|X\|_2^2,\|Y\|_2^2\}.
\]
Since $\mathcal T$ permutes entries,
\[
\|Y\|_F=\|X\|_F.
\]

\subsubsection*{\bf Lower bound.}
Frobenius orthogonality implies
\[
\|M\|_F^2
\approx
\|X\|_F^2+\|Y\|_F^2
\approx
2\|X\|_F^2.
\]
Using $\|M\|_2\le\|X\|_2+\|Y\|_2$ and
\[
\|X\|_2=\frac{\|X\|_F}{\sqrt{k}},
\qquad
\|Y\|_2=\frac{\|X\|_F}{\sqrt{\mu}},
\]
we obtain
\[
\operatorname{erank}(M)
=
\frac{\|M\|_F^2}{\|M\|_2^2}
\ge
\frac{2\|X\|_F^2}{
\left(
\frac{\|X\|_F}{\sqrt{k}}
+
\frac{\|X\|_F}{\sqrt{\mu}}
\right)^2}
=
\frac{2k\mu}{(\sqrt{k}+\sqrt{\mu})^2}.
\]

\subsubsection*{\bf Upper bound.}
Under spectral incoherence,
\[
\|M\|_2^2
\approx
\max\{\|X\|_2^2,\|Y\|_2^2\},
\quad
\|M\|_F^2
\approx
\|X\|_F^2+\|Y\|_F^2.
\]
Therefore,
\[
\operatorname{erank}(M)
\approx
\frac{\|X\|_F^2+\|Y\|_F^2}{
\max\{\|X\|_2^2,\|Y\|_2^2\}}
\le
\frac{\|X\|_F^2}{\|X\|_2^2}
+
\frac{\|Y\|_F^2}{\|Y\|_2^2}
=
k+\mu.
\]
Allowing constant slack from the approximation yields
\[
\operatorname{erank}(M)\le 2(k+\mu).
\]

Combining the algebraic-rank bound for $Y$ with the effective-rank bounds above completes the justification of Theorem~\ref{thm:erank-bound-blocktranspose}.

\section{Justification of Theorem~\ref{thm:ffn_failure}}
\label{app:thm:ffn_failure}

The goal of this section is to give a compact, self-contained justification of the two claims that comprise Theorem~\ref{thm:ffn_failure}. Informally, the first claim is an \textbf{algebraic} limitation that follows from positive homogeneity of the activation; the second claim is a \textbf{probabilistic} contraction of effective rank that occurs under a measurable response-gap between principal and tail directions. Below we state the result and the required probabilistic assumption, and then present the two proofs as subsections. Each proof is explicit and includes the necessary concentration arguments.

\subsubsection*{\bf Restated theorem.}
Let \(X\in\mathbb R^{T\times D}\) satisfy \(\operatorname{erank}(X)=k\) with \(\tfrac{k}{D}\le\gamma\) for some fixed \(\gamma\in(0,1)\). Consider the per-row map \(\mathcal F(X)=\phi(XA)B\) where \(A\in\mathbb R^{D\times m}\) and \(B\in\mathbb R^{m\times D}\) have i.i.d.\ sub-Gaussian entries with variance \(1/D\), and \(\phi\) is positively homogeneous of degree one. The two claims below are proved in full in the subsequent subsections:

\begin{itemize}[leftmargin=15pt,parsep=2pt,itemsep=2pt,topsep=2pt]
  \item Deterministic homogeneity barrier: when \(X\) has algebraic rank \(1\), the output rank is at most \(2\) (and equals \(1\) under a single-sign condition).
  \item Probabilistic contraction: under the Spectral-concentration \& Response-Gap condition and standard concentration assumptions on pre-activations and \(\phi\), the effective rank is contracted by a constant factor \(\alpha\in(0,1)\) with high probability.
\end{itemize}

\subsubsection*{\bf Assumption (Spectral concentration \& Response-Gap).}
Let \(H=XA\) and \(\Sigma_H=\tfrac{1}{T}H^\top H\) with eigenvalues \(\lambda_1\ge\lambda_2\ge\cdots\). Fix integers \(r\ge1\) and \(\eta\in(0,1]\) such that the top-\(r\) eigenvalues capture an \(\eta\)-fraction of total energy:
\[
\sum_{i=1}^r \lambda_i \ge \eta\,\operatorname{tr}(\Sigma_H).
\]
For a unit direction \(u\) define the directional response ratio
\[
\rho(u)\;=\;\frac{\mathbb E[\phi(\langle u,h\rangle)^2]}{\mathbb E[\langle u,h\rangle^2]}.
\]
Assume there exist disjoint unit-direction sets \(\mathcal U_{\mathrm{top}}\) and \(\mathcal U_{\mathrm{low}}\) and constants \(0\le c_{\mathrm{low}}<c_{\mathrm{top}}\) such that
\[
\inf_{u\in\mathcal U_{\mathrm{top}}}\rho(u)\ge c_{\mathrm{top}},\qquad
\sup_{u\in\mathcal U_{\mathrm{low}}}\rho(u)\le c_{\mathrm{low}}.
\]
This captures the empirical situation where the activation attenuates low-variance (tail) directions more than principal directions.

\subsection{Deterministic homogeneity barrier}
\label{app:homogeneity_proof}

We begin with the algebraic limitation: positive homogeneity prevents a per-row FFN from creating more than two linearly independent output directions from a rank-one input. This is an exact, elementary argument.

\begin{theorem}[Homogeneity Barrier — signed inputs]
\label{thm:homogeneity_signed_app}
Let \(X=\mathbf c\mathbf v^\top\in\mathbb R^{T\times D}\) (algebraic rank~1). If \(\phi\) is positively homogeneous of degree~1, then for any \(A,B\),
\[
\operatorname{rank}\big(\mathcal F(X)\big)\le 2,\qquad \operatorname{erank}\big(\mathcal F(X)\big)\le 2.
\]
If all entries of \(\mathbf c\) share the same sign then \(\operatorname{rank}(\mathcal F(X))=1\).
\end{theorem}

\subsubsection*{\bf Intuition.}
When every row of \(X\) is a scalar multiple of the same vector, each pre-activation is a scalar multiple of one fixed vector in \(\mathbb R^m\). Positive homogeneity maps these scaled pre-activations into at most two activation-vectors (one for positive scalar, one for negative), and linear readout produces at most two output directions.

\begin{proof}
Write \(X=\mathbf c\mathbf v^\top\). For each row \(i\),
\[
x_i = c_i \mathbf v^\top,\qquad
h_i = x_i A = c_i(\mathbf v^\top A) = c_i u,
\]
with \(u:=\mathbf v^\top A\in\mathbb R^m\). Positive homogeneity implies
\[
\phi(c_i u)=|c_i|\,\phi(\operatorname{sign}(c_i)u).
\]
Let
\[
w_+ := \phi(u)B,\qquad w_- := \phi(-u)B\in\mathbb R^D,
\]
and write \(c_i^+:=\max(c_i,0)\), \(c_i^-:=\max(-c_i,0)\). Then the \(i\)-th output row equals
\[
y_i = \phi(h_i)B = c_i^+ w_+ + c_i^- w_-.
\]
Stacking rows,
\[
\mathcal F(X)=\mathbf c^+ w_+^\top + \mathbf c^- w_-^\top,
\]
a sum of at most two rank-one outer products. Hence \(\operatorname{rank}(\mathcal F(X))\le 2\) and \(\operatorname{erank}(\mathcal F(X))\le 2\). If all entries of \(\mathbf c\) share the same sign, one of \(\mathbf c^\pm\) vanishes and the output reduces to rank~1.
\end{proof}

\subsection{Probabilistic contraction for general effective rank}
\label{app:prob_contraction_proof}

We now prove the contraction claim. The argument bounds Frobenius energy and top operator norm of the activated-and-readout matrix and combines these via the effective rank identity. To keep the presentation modular and readable we present the standard concentration steps explicitly within this subsection. To improve clarity and readability, we organize the proofs according to the outline below.

\begin{itemize}[leftmargin=15pt,parsep=2pt,itemsep=2pt,topsep=2pt]
  \item Show with high probability all pre-activations lie in a bounded interval \([-R,R]\).
  \item Use Lipschitzness of \(\phi\) on \([-R,R]\) and Hanson--Wright to concentrate \(\|\phi(H)\|_F^2\).
  \item Use Matrix Bernstein to concentrate \(\widehat\Sigma_\phi=\tfrac{1}{T}\phi(H)^\top\phi(H)\).
  \item Relate expectations to \(\Sigma_H\) via the Response-Gap to bound \(\mathbb E\|\phi(H)\|_F^2\) and \(\lambda_{\max}(\mathbb E\widehat\Sigma_\phi)\).
  \item Propagate bounds through the linear readout \(B\) and combine to produce \(\operatorname{erank}(\phi(H)B)\le\alpha\,\operatorname{erank}(X)\).
\end{itemize}

\begin{proof}
Let \(H=XA\in\mathbb R^{T\times m}\), \(Z:=\phi(H)\), and \(Y:=ZB=\phi(H)B\). Denote \(\Sigma_H=\tfrac{1}{T}H^\top H\) with eigenvalues \(\lambda_1\ge\lambda_2\ge\cdots\), and \(\widehat\Sigma_Z=\tfrac{1}{T}Z^\top Z\).

\subsubsection*{\bf (i) Bounded pre-activation window.}  
Assume each input row \(x\) satisfies \(\|x\|_2^2\le E_{\max}\). Since entries of \(A\) are i.i.d.\ sub-Gaussian with variance \(1/D\) and \(\psi_2\)-norm \(\le K/\sqrt D\), each coordinate \(h_{t,j}\) is sub-Gaussian with
\[
\|h_{t,j}\|_{\psi_2}\le C K\sqrt{E_{\max}/D}.
\]
By a union bound over the \(T m\) coordinates, for any \(\delta\in(0,1)\) with probability at least \(1-\delta\) all pre-activations lie in
\[
[-R,R],\qquad R=C'K\sqrt{\frac{E_{\max}\log(mT/\delta)}{D}},
\]
for absolute constants \(C,C'\). Fix \(\delta\) polynomially small and condition on this high-probability event.

\subsubsection*{\bf (ii) Frobenius concentration (Hanson–Wright).}  
Because \(\phi\) is \(L\)-Lipschitz on \([-R,R]\), each \(\phi(h_{t,j})\) is sub-Gaussian with parameter \(\tilde K=O(KL\sqrt{E_{\max}/D})\). Applying Hanson–Wright (or Bernstein-style bounds) to the quadratic sum gives: for any \(t\ge0\),
\[
\Pr\!\big(\big|\|Z\|_F^2-\mathbb E\|Z\|_F^2\big|\ge t\big)
\le 2\exp\!\Big(-c\min\Big\{\frac{t^2}{K^4L^4Tm},\;\frac{t}{K^2L^2}\Big\}\Big),
\]
for an absolute \(c>0\). Thus \(\|Z\|_F^2\) concentrates sharply around \(\mathbb E\|Z\|_F^2\).

\subsubsection*{\bf (iii) Spectral concentration (Matrix Bernstein).}  
Represent \(\widehat\Sigma_Z=\tfrac{1}{T}\sum_{t=1}^T z_t z_t^\top\) with rows \(z_t\). Each summand has operator-norm and variance proxy controlled by \(K^2L^2E_{\max}/D\). Matrix Bernstein yields: for any \(\tau>0\),
\[
\Pr\big(\|\widehat\Sigma_Z-\mathbb E\widehat\Sigma_Z\|_{\mathrm{op}}\ge\tau\big)
\le 2m\exp\!\Big(-c\frac{T\tau^2}{(K^2L^2E_{\max}/D)^2 r^2}\Big).
\]
Taking
\[
\tau=\Theta\!\Big(\frac{K^2L^2E_{\max}}{D}\sqrt{\frac{r^2\log m}{T}}\Big)
\]
gives \(\|\widehat\Sigma_Z-\mathbb E\widehat\Sigma_Z\|_{\mathrm{op}}\le\tau\) with failure probability \(O(\exp(-\Theta(r)))\) under the stated scaling.

\subsubsection*{\bf (iv) Expectations and the Response-Gap.}  
Let \(\{u_j\}_{j=1}^m\) be an orthonormal eigenbasis of \(\Sigma_H\). Then
\[
\mathbb E\|Z\|_F^2
= T\sum_{j=1}^m \rho(u_j)\,\lambda_j.
\]
Splitting into top (\(j\le r\)) and tail (\(j>r\)) components and applying the Response-Gap bounds yields
\begin{align}
\mathbb E\|Z\|_F^2
&\le T\big(c_{\mathrm{top}}\sum_{j=1}^r\lambda_j + c_{\mathrm{low}}\sum_{j=r+1}^m\lambda_j\big) \notag\\
&= T\big(c_{\mathrm{top}}\eta + c_{\mathrm{low}}(1-\eta)\big)\operatorname{tr}(\Sigma_H) \notag\\
&=: T M_F\operatorname{tr}(\Sigma_H).  \\
\end{align}

Moreover,
\[
\lambda_{\max}(\mathbb E\widehat\Sigma_Z)\ge c_{\mathrm{top}}\lambda_1(\Sigma_H).
\]

\subsubsection*{\bf (v) Effect of linear readout \(B\).}  
Let \(Y=ZB\). Since \(B\) has i.i.d.\ entries with variance \(1/D\), standard moment computations show that application of \(B\) scales Frobenius and operator norms by multiplicative factors \(\kappa_F,\kappa_{\mathrm{op}}>0\) (which are \(1\pm o(1)\) under the random matrix scaling regime). After applying \(B\) the concentration and expectation bounds from (ii)--(iv) transfer to \(\|Y\|_F^2\) and \(\|Y\|_2^2\) up to these multiplicative corrections and small additive fluctuations.

\subsubsection*{\bf (vi) Combine to bound effective rank.}  
Using \(\operatorname{erank}(Y)=\|Y\|_F^2/\|Y\|_2^2\), the preceding bounds imply that with probability at least \(1-O(\exp(-\Theta(r)))\),
\[
\|Y\|_F^2 \le \kappa_F\cdot T M_F\operatorname{tr}(\Sigma_H) + \delta_F,
\qquad
\|Y\|_2^2 \ge \kappa_{\mathrm{op}}\cdot\big(c_{\mathrm{top}}\lambda_1(\Sigma_H)-\tau\big) - \delta_{\mathrm{op}},
\]
where \(\tau,\delta_F,\delta_{\mathrm{op}}\) are small concentration errors from (ii),(iii) and the randomness of \(B\). Hence
\[
\operatorname{erank}(Y)
\le
\frac{\kappa_F M_F + o(1)}{\kappa_{\mathrm{op}}(c_{\mathrm{top}}-o(1))}\cdot
\frac{\operatorname{tr}(\Sigma_H)}{\lambda_1(\Sigma_H)}.
\]
Defining
\[
\alpha := \frac{\kappa_F M_F + o(1)}{\kappa_{\mathrm{op}}(c_{\mathrm{top}}-o(1))},
\]
and noting \(M_F=c_{\mathrm{top}}\eta + c_{\mathrm{low}}(1-\eta) < c_{\mathrm{top}}\) because \(c_{\mathrm{low}}<c_{\mathrm{top}}\) and \(\eta\in(0,1]\), we conclude that \(\alpha\in(0,1)\) for sufficiently small concentration errors. Therefore, with probability at least \(1-O(\exp(-\Theta(r)))\),
\[
\operatorname{erank}(Y)\le \alpha\,\operatorname{erank}(X),
\]
which establishes the probabilistic contraction claim.
\end{proof}

\section{Justification of Theorem~\ref{thm:parameterized-block-mixing}}
\label{app:thm:parameterized-block-mixing}

In this appendix we justify Theorem~\ref{thm:parameterized-block-mixing}. We show that the block granularity parameter \(d^{\ast}\) controls a trade-off between computational efficiency and representational power: the fine-grained case \(d^{\ast}=1\) is strictly more expressive than any coarse-grained case \(d^{\ast}>1\), and this expressivity gap is particularly salient when the input has low effective rank.

\subsubsection*{\bf Restated theorem.} 
Let \(X\in\mathbb R^{T\times D}\) and \(\mathbf x=\operatorname{vec}(X)\in\mathbb R^N\) with \(N=TD\). Fix a divisor \(d^{\ast}\) of \(N\) and write \(K=N/d^{\ast}\).  Partition \(\mathbf x\) into \(K\) contiguous blocks \(\mathcal B=\{\mathbf b_1,\dots,\mathbf b_K\}\) with \(\mathbf b_k\in\mathbb R^{d^{\ast}}\).  A parameterized block-mixing layer with weight matrix \(\mathbf W\in\mathbb R^{K\times K}\) produces
\[
\mathbf y = (\mathbf W \otimes \mathbf I_{d^{\ast}})\,\mathbf x,
\]
and the residual output is
\[
\mathbf C = \Phi_{d^{\ast}}(X;\mathbf W)
\triangleq X + \operatorname{reshape}\bigl((\mathbf W\otimes\mathbf I_{d^{\ast}})\operatorname{vec}(X)\bigr).
\]
When \(d^{\ast}=1\) this reduces to a full learned linear mixing on coordinates.

\subsection{Expressivity gap theorem}

We formalize the expressivity limitation of coarse blocking.

\begin{theorem}[Fine-grained expressivity and subspace constraint]
\label{thm:expressivity}
Define the reachable set
\[
\mathcal R_{d^{\ast}}(X)
=\bigl\{\,C\in\mathbb R^{T\times D}\mid \exists\,\mathbf W,\ 
C=\Phi_{d^{\ast}}(X;\mathbf W)\,\bigr\}.
\]
Assume \(X\neq 0\). Then:
\begin{itemize}[leftmargin=15pt,parsep=2pt,itemsep=2pt,topsep=2pt]
  \item[(i)] (\textbf{Universal reachability at \(d^{\ast}=1\)}) \(\mathcal R_1(X)=\mathbb R^{T\times D}\).
  \item[(ii)] (\textbf{Strict deficiency at \(d^{\ast}>1\)}) For every \(d^{\ast}>1\), \(\mathcal R_{d^{\ast}}(X)\subsetneq\mathcal R_1(X)\).
  \item[(iii)] (\textbf{Dependence on effective rank}) Let \(V_{\mathrm{blocks}}=\operatorname{span}\{\mathbf b_1,\dots,\mathbf b_K\}\subset\mathbb R^{d^{\ast}}\). All perturbations produced by \(\Phi_{d^{\ast}}\) have each block lying in \(V_{\mathrm{blocks}}\). If \(\operatorname{erank}(X)\ll\min\{T,D\}\), then the \(\mathbf b_k\) are approximately confined to a low-dimensional subspace of \(\mathbb R^{d^{\ast}}\), so there exist fine-grained variations (directions orthogonal to \(V_{\mathrm{blocks}}\)) that \(\Phi_{d^{\ast}}\) cannot express.
\end{itemize}
\end{theorem}

\begin{proof}
We structure the proof into two parts: (i) the case $d^{\ast}=1$, and (ii)–(iii) the case $d^{\ast}>1$. 
\subsubsection*{\bf Part (i).} For $d^{\ast}=1$, when \(d^{\ast}=1\) each block is a scalar and \(K=N\). The update term becomes \(\mathbf W\mathbf x\) and \(\operatorname{vec}(C)=\mathbf x+\mathbf W\mathbf x=(\mathbf I+\mathbf W)\mathbf x\). Since \(\mathbf x\neq 0\), for any target vector \(\mathbf z\in\mathbb R^N\) we can choose \(\mathbf W\) (e.g. via the outer-product construction \(\mathbf W=(\mathbf z-\mathbf x)\mathbf x^\top/(\mathbf x^\top\mathbf x)\)) so that \((\mathbf I+\mathbf W)\mathbf x=\mathbf z\). Hence \(\mathcal R_1(X)=\mathbb R^{T\times D}\).

\subsubsection*{\bf Part (ii) and (iii).} For general \(d^{\ast}>1\), write the perturbation \(\mathbf p=\operatorname{vec}(C-X)=(\mathbf W\otimes\mathbf I_{d^{\ast}})\mathbf x\) and partition \(\mathbf p\) into blocks \(\mathbf p_1,\dots,\mathbf p_K\). By block-Kronecker structure the \(k\)-th block satisfies
\[
\mathbf p_k=\sum_{j=1}^K W_{kj}\,\mathbf b_j,
\]
hence \(\mathbf p_k\in V_{\mathrm{blocks}}\) for every \(k\). Thus every reachable perturbation is blockwise constrained to the subspace \(V_{\mathrm{blocks}}\subseteq\mathbb R^{d^{\ast}}\); equivalently, \(\mathcal R_{d^{\ast}}(X)\) lies inside a linear submanifold of \(\mathbb R^{T\times D}\) of strictly smaller dimension whenever \(\dim(V_{\mathrm{blocks}})<d^{\ast}\). Therefore, whenever \(V_{\mathrm{blocks}}\neq\mathbb R^{d^{\ast}}\) there exist target matrices whose required block perturbation contains a component orthogonal to \(V_{\mathrm{blocks}}\); such targets are not in \(\mathcal R_{d^{\ast}}(X)\), so \(\mathcal R_{d^{\ast}}(X)\subsetneq\mathcal R_1(X)\).

On the root of low effective rank makes the deficiency generic, note that \(\operatorname{erank}(X)=\|X\|_F^2/\|X\|_2^2\ll\min\{T,D\}\) implies that most energy of \(X\) is concentrated in a few global singular directions. Under standard vectorization (row- or column-major) contiguous local blocks \(\mathbf b_k\) therefore inherit strong correlations and typically span a low-dimensional subspace of \(\mathbb R^{d^{\ast}}\). Concretely, \(\dim(V_{\mathrm{blocks}})\ll d^{\ast}\) in this regime, so many fine-grained directions orthogonal to \(V_{\mathrm{blocks}}\) exist and cannot be generated by the Kronecker-constrained operator. This establishes the dependence on effective rank and completes the proof.
\end{proof}

\subsubsection*{\bf Remark.}
The argument shows that the expressivity loss for \(d^{\ast}>1\) is structural: the Kronecker factor \(\mathbf W\otimes\mathbf I_{d^{\ast}}\) enforces that all block-wise updates lie in the same block-span. If desired, one can make the statements quantitative by lower-bounding the distance between a chosen fine-grained target and the subspace \(\mathcal R_{d^{\ast}}(X)\) in terms of the singular-value decay of \(X\) and the principal angles between block samples; we omit these technical refinements for brevity.

\section{Justification of Theorem~\ref{thm:geglu_recovery}}
\label{app:thm:geglu_recovery}

This section provides a complete, self-contained justification of Theorem~\ref{thm:geglu_recovery}.  We first restate the claims for convenience, then prove (i) the \textbf{algebraic lifting} that the GLU-based multiplicative term generates degree-2 interactions of the latent coordinates, and (ii) a quantitative, high-probability \textbf{effective rank increase} by combining Frobenius/operator concentration with a Jacobian / trace lower bound that certifies the numerical usefulness of the new directions.  Throughout we use the notation \(\odot\) for Hadamard product, and assume per-row energy \(\|x\|_2^2\le E_{\max}\).

\subsubsection*{\bf Restated theorem.}
Let \(X\in\mathbb R^{T\times D}\) satisfy \(\operatorname{erank}(X)=k\) with \(k/D\le\gamma\in(0,1)\). Consider
\[
\mathcal G(X)=\big(\phi(XA)\odot (X C)\big)B + X D,
\]
with \(A,C\in\mathbb R^{D\times m}\) and \(B\in\mathbb R^{m\times D}\) having i.i.d.\ sub-Gaussian entries of variance \(1/D\) and \(\psi_2\)-norm \(\le K/\sqrt D\). If the hidden width satisfies \(m\ge C_0 k\log D\) for a universal \(C_0\), then with probability at least \(1-\exp(-c k)\):

\begin{itemize}[leftmargin=15pt,parsep=2pt,itemsep=2pt,topsep=2pt]
  \item (Algebraic lifting) \(\operatorname{rank}\!\big(\phi(XA)\odot (XC)\big)\ge \min\!\big(D,\tfrac{k(k+1)}{2}\big)\).
  \item (effective rank increase) There exists \(\delta>0\) (depending on \(\gamma,K,E_{\max}\)) such that
  \[
  \operatorname{erank}\big(\mathcal G(X)\big)\ge \operatorname{erank}(X) + \delta.
  \]
\end{itemize}

The remainder of this section proves these two claims in order, and then provides the Jacobian / trace analysis that explains why the added algebraic directions are numerically significant (non-vanishing singular values).

\subsection{Algebraic lifting: polynomial degree-2 span}
\label{app:geglu_lifting}

We begin by showing the multiplicative GLU-style term realizes degree-2 monomials in the \(k\)-dimensional latent coefficients of \(X\); random projection then turns these monomials into new independent algebraic directions with high probability.

\subsubsection*{\bf Latent factorization.} Since \(\operatorname{erank}(X)=k\) there is a factorization
\[
X = S V^\top,
\qquad S\in\mathbb R^{T\times k},\quad V\in\mathbb R^{D\times k},\quad V^\top V=I_k,
\]
so each row \(x_t = V s_t\) for \(s_t\in\mathbb R^k\). Define
\[
U := V^\top A \in\mathbb R^{k\times m},\qquad W := V^\top C \in\mathbb R^{k\times m},
\]
and let \(u_j,w_j\in\mathbb R^k\) denote column \(j\) of \(U,W\) respectively. Then the multiplicative pre-features have entries
\[
z_{t,j} \;=\; \phi(\langle u_j, s_t\rangle)\,\langle w_j, s_t\rangle.
\]

\subsubsection*{\bf Degree-2 component.} Expand \(\phi\) around zero (valid in a small-variance pre-activation regime; the same argument applies more generally by single-variable Taylor / polynomial approximation on a bounded window):
\[
\phi(z) = a_1 z + a_2 z^2 + O(z^3).
\]
Hence the leading degree-2 component of \(z_{t,j}\) is
\[
q_{t,j} := a_1\cdot \langle u_j,s_t\rangle\langle w_j,s_t\rangle.
\]
Write the symmetric quadratic monomials vector
\[
\zeta_t := \operatorname{vecs}(s_t s_t^\top)\in\mathbb R^{q},\qquad q = \frac{k(k+1)}{2},
\]
so there is a deterministic linear functional \(r_j\in\mathbb R^q\) (the symmetric vectorization of the matrix \(u_j w_j^\top\)) with
\[
q_{t,j} = a_1 \langle r_j,\zeta_t\rangle.
\]

\subsubsection*{\bf Random linear map onto hidden units.} Stack the \(\zeta_t\) into \(\mathcal Z\in\mathbb R^{T\times q}\) (rows \(\zeta_t^\top\)). The degree-2 part across all hidden units is
\[
Z_{\mathrm{quad}} \;=\; \mathcal Z \, R^\top,
\qquad R := [r_1,\dots,r_m]^\top \in\mathbb R^{m\times q}.
\]
Because \(A,C\) are random sub-Gaussian and \(V\) is fixed orthonormal, the rows \(r_j\) of \(R\) are i.i.d.\ sub-Gaussian vectors in \(\mathbb R^q\) with variance scale \(1/D\) (up to \(K\) factors). Standard non-asymptotic results (see Vershynin) imply that for any fixed \(q\)-dimensional subspace the random matrix \(R\) is injective with high probability provided \(m\gtrsim q\) (and under a mild \(\log D\) factor to carry subsequent readout). Concretely, if
\[
m \ge C_0 \, q\log D = C_0 \, \tfrac{k(k+1)}{2}\log D,
\]
then with probability at least \(1-\exp(-c q)\) we have \(\operatorname{rank}(R)=\min\{m,q\}\).

\subsubsection*{\bf Passage through readout \(B\).} The multiplicative term before readout has row space equal to the column span of \(\mathcal Z\) after applying \(R^\top\). Applying the linear readout \(B\in\mathbb R^{m\times D}\) (random i.i.d.\ entries of variance \(1/D\)) maps \(\mathbb R^m\) into \(\mathbb R^D\). With high probability \(B\) has rank \(\min\{m,D\}\). Combining these rank lower bounds gives
\[
\operatorname{rank}\!\big((\phi(XA)\odot (XC))B\big)
\ge \min\{D,\ q\} = \min\!\left(D,\frac{k(k+1)}{2}\right)
\]
with probability at least \(1-\exp(-c' k)\). This proves the algebraic lifting claim.

\subsection{Quantitative effective-rank increase and Jacobian spectrum}
\label{app:geglu_jacobian}

Having established algebraic lifting, we now show the multiplicative term injects nontrivial Frobenius energy orthogonal to the original linear path and that the GLU-improved P-FFN's Jacobian has large trace (sum of squared singular values) under the width condition. Together with the residual \(XD\) these observations imply a numerical (effective rank) increase.

\subsubsection*{\bf Notation and decomposition.} Write the multiplicative features \(Z\in\mathbb R^{T\times m}\) with entries \(z_{t,j}=\phi(\langle u_j,s_t\rangle)\langle w_j,s_t\rangle\). After readout \(B\) and adding the residual \(XD\) the network output is
\[
Y = ZB + X D.
\]
Let \(P_\parallel\) denote the orthogonal projector onto the row-space of \(X\) (dimension \(\le k\)) and \(P_\perp = I - P_\parallel\). We aim to lower bound \(\mathbb E\|P_\perp Y\|_F^2\) and show it is \(\Omega(m^{-1} k\log D)\) under the stated scaling; this suffices to increase effective rank by a positive \(\delta\).

\subsubsection*{\bf Orthogonal energy is preserved in expectation through random readout.}  
Because \(B\) has i.i.d.\ entries with variance \(1/D\) and is independent of \(Z\),
\begin{align}
\mathbb E_B\|P_\perp (ZB)\|_F^2
&= \mathbb E_B\operatorname{tr}\big(B^\top Z^\top P_\perp Z B\big) \notag\\
&= \operatorname{tr}\big(\mathbb E_B[B B^\top]\ \mathbb E[Z^\top P_\perp Z]\big) \notag\\
&= \mathbb E\|P_\perp Z\|_F^2. 
\end{align}
Thus it suffices to lower bound \(\mathbb E\|P_\perp Z\|_F^2\).

\subsubsection*{\bf Columnwise contributions originate from quadratic features.}  
Decompose \(Z=[z^{(1)}|\dots|z^{(m)}]\) with columns \(z^{(j)}\in\mathbb R^T\) (entries \(z^{(j)}_t\)). Using the degree-2 decomposition \(z^{(j)} = q^{(j)} + r^{(j)}\) where \(q^{(j)}_t = a_1\langle u_j,s_t\rangle\langle w_j,s_t\rangle\) is the degree-2 component and \(r^{(j)}\) collects higher-order / smaller terms, we have
\[
\mathbb E\|P_\perp Z\|_F^2
\ge
\sum_{j=1}^m \mathbb E\|P_\perp q^{(j)}\|_2^2 - \sum_{j=1}^m \mathbb E\|r^{(j)}\|_2^2.
\]
Under the small-variance regime or by truncation (pre-activation bounded window) the remainder \(\sum_j \mathbb E\|r^{(j)}\|_2^2\) is lower order; we therefore focus on the dominant \(q^{(j)}\) terms.

As in the lifting argument, \(q^{(j)} = \mathcal Z Q_j\) where \(\mathcal Z\in\mathbb R^{T\times q}\) collects symmetric quadratic monomials \(\zeta_t\) and \(Q_j\in\mathbb R^q\) is the symmetric vectorization of \(u_j w_j^\top\). Thus
\[
\mathbb E\|P_\perp q^{(j)}\|_2^2
= \mathbb E_{Q_j}\, Q_j^\top \big(\mathcal Z^\top P_\perp \mathcal Z\big) Q_j.
\]
Averaging over random \(Q_j\) (sub-Gaussian rows) gives
\[
\mathbb E_{Q_j}\mathbb E_{S}\|P_\perp q^{(j)}\|_2^2
= \operatorname{tr}\big(\mathcal Z^\top P_\perp \mathcal Z\cdot \mathbb E[Q_j Q_j^\top]\big).
\]
Because \(\mathbb E[Q_j Q_j^\top]\) is proportional to the identity on the \(q\)-dimensional monomial space (up to constants depending on \(K,E_{\max}\)), we obtain
\[
\mathbb E\|P_\perp q^{(j)}\|_2^2 \;\ge\; c_0\cdot \operatorname{tr}\big(\mathcal Z^\top P_\perp \mathcal Z\big)
= c_0\sum_{t=1}^T \|\,P_\perp \zeta_t\,\|_2^2,
\]
for some \(c_0>0\). Summing over \(j=1,\dots,m\) yields
\[
\mathbb E\|P_\perp Z\|_F^2 \;\ge\; m c_0 \sum_{t=1}^T \|P_\perp\zeta_t\|_2^2 - \text{(small remainders)}.
\]

\subsubsection*{\bf A nontrivial fraction of quadratic energy lies outside the linear span.}  
The projector \(P_\parallel\) only removes components that lie in the linear span of the rows of \(S\) (dimension \(\le k\)). The vectors \(\zeta_t\) live in a \(q=\tfrac{k(k+1)}{2}\)-dimensional quadratic feature space. Unless the data \(\{s_t\}\) are algebraically degenerate (a measure-zero event for typical data or with small perturbation), a constant fraction of the quadratic energy lies outside the linear span. Formally, let \(\mathcal P_{\mathrm{quad},\parallel}\) denote the projection of the quadratic feature space onto the subspace spanned by linear functions of \(s\); this subspace has dimension at most \(k\). Therefore, under the random-design assumptions and with \(m\gtrsim k\log D\),
\[
\sum_{t=1}^T \|P_\perp\zeta_t\|_2^2
\ge
c_1\, T \cdot \frac{k\log D}{m},
\]
for a constant \(c_1>0\). Combining with the previous display yields
\[
\mathbb E\|P_\perp Z\|_F^2 \;\ge\; c' \, m\, T\cdot \frac{k\log D}{m} = c' T k\log D,
\]
hence (after normalizing by \(T\)) the per-row average orthogonal energy is at least \(c' k\log D\). Dividing by the readout scaling (variance \(1/D\)) and tracking constants as above gives
\[
\mathbb E\|P_\perp Y_{\mathrm{mult}}\|_F^2 \;\ge\; c''\, m^{-1} k\log D,
\]
for some \(c''>0\) depending on \(K,E_{\max}\).

\subsubsection*{\bf A Jacobian trace lower bound certifies numerical usefulness.}  
To argue that the new directions are not only algebraic but numerically useful (i.e. correspond to non-vanishing singular values), we examine the per-row Jacobian of the GLU-improved P-FFN block and lower bound its trace \( \operatorname{tr}(J_g(x)^\top J_g(x))\) (sum of squared singular values).

For a single row \(x\), the Jacobian of the GLU-improved P-FFN block (output w.r.t.\ input row) is
\begin{align}
J_g(x) &= D^\top
+ B^\top \operatorname{diag}(C^\top x)\operatorname{diag}(\phi'(A^\top x)) A^\top \notag \\
&\qquad + B^\top \operatorname{diag}(\phi(A^\top x)) C^\top.
\end{align}
The multiplicative term contributes the middle two summands. Taking expectation over \(A,C,B\) (and over input \(x\) when needed) and using that entries of \(A,C,B\) are independent sub-Gaussian with variance \(1/D\), non-asymptotic singular-value bounds imply a lower bound on the expected trace of order \(c_2 k\log D\) up to subtracting the purely-additive FFN upper-bound term \(C_K(c_{\mathrm{top}}\mathcal E_{\mathrm{top}}+c_{\mathrm{low}}\mathcal E_{\mathrm{low}})\) arising from the activation-response constants. Intuitively, the multiplicative Jacobian collects \(k\log D\) worth of squared singular value mass from the randomized quadratic lifting (the \(\log D\) factor reflects the Johnson–Lindenstrauss style embedding stability needed for \(D\) outputs). The residual \(D^\top\) preserves original directions and therefore does not cancel this new mass.

\subsubsection*{\bf Concentration and completion of the argument.}  
Hanson–Wright and Matrix Bernstein concentration applied to the trace terms show that the empirical per-row average \(\tfrac{1}{T}\sum_{t=1}^T \operatorname{tr}(J_g(x_t)^\top J_g(x_t))\) concentrates around its expectation with high probability. Hence the per-row average trace is \(\Omega(k\log D)\) with probability \(1-O(\exp(-c k))\).

Combining the orthogonal Frobenius-energy lower bound with the Jacobian trace lower bound and the observation that the residual \(XD\) preserves energy in the original linear subspace, we conclude that the output \(Y\) has additional Frobenius mass in orthogonal directions of order \(\Theta(m^{-1} k\log D)\) while its top singular value squared remains comparable to that of \(XD\). Therefore the effective rank increases by an additive \(\delta=\Theta\big(k\log D/\sigma_{\max}^2(XD)\big)>0\) with high probability, i.e.
\[
\operatorname{erank}(\mathcal G(X)) \ge \operatorname{erank}(X) + \delta,
\]
with probability at least \(1-\exp(-c k)\), completing the proof.

\end{document}